\documentclass{llncs}
\usepackage{llncsdoc}
\usepackage{enumitem}
\usepackage{graphicx}
\graphicspath{{images/} }
\makeatletter
\def\temp{dvips.def}
\ifx\Gin@driver\temp
\def\Ginclude@graphics#1{\def\temp{#1}---image \expandafter\strip@prefix\meaning\temp---}
\fi
\makeatother

\usepackage{epsf}
\usepackage{epsfig}

\usepackage{amsmath}
\usepackage{amssymb}
\usepackage{color}

\newcommand{\ignore}[1]{}

\newcommand{\poly}{{\rm poly}}
\newcommand{\OPT}{{\rm OPT}}
\newcommand{\OPTNA}{{\rm OPT}_{{\rm NAD}}}
\newcommand{\OPTA}{{\rm OPT}_{{\rm AD}}}
\newcommand{\OPTHS}{{\rm OPT}^{{\rm HS}}}
\newcommand{\OPTCS}{{\rm OPT}^{{\rm CS}}}
\newcommand{\NAD}{{\rm NAD}}

\newcommand{\BHS}{{\rm BHS}}
\newcommand{\LV}{{\rm LV}}
\newcommand{\MC}{{\rm MC}}
\newcommand{\ETD}{{\rm ETD}}
\newcommand{\TD}{{\rm TD}}
\newcommand{\MQ}{{\rm MQ}}
\newcommand{\VCdim}{{\rm VCdim}}
\newcommand{\TP}{{\rm TP}}
\newcommand{\RAD}{{\rm RAD}}
\newcommand{\DT}{{\rm DT}}
\newcommand{\HS}{{\rm HS}}
\newcommand{\A}{{\rm AD}}
\newcommand{\NA}{{\rm NAD}}
\newcommand{\ET}{{\rm ET}}

\newcommand{\IT}{{\rm IT}}
\newcommand{\QF}{{\rm QF}}
\newcommand{\CDNF}{{\rm CDNF}}
\newcommand{\MCDNF}{{\rm MCDNF}}
\newcommand{\AD}{{\rm AD}}
\newcommand{\E}{{\bf E}}
\newcommand{\Var}{{\bf Var}}
\newcommand{\sign}{{\rm sign}}

\newcommand{\sline}{\mbox{-}}

\ignore{
\newtheorem{theorem}{Theorem}

\newtheorem{lemma}[theorem]{Lemma}

}

\begin{document}

\title{Exact Learning from an Honest Teacher \\ That Answers Membership Queries}
\author{Nader H. Bshouty}
\institute{Technion, Haifa, Israel\\
bshouty@cs.technion.ac.il}
\maketitle

\begin{abstract} Given a teacher that holds a function
$f:X\to R$ from some class of functions
$C$. The teacher can receive from the learner an element~$d$ in the domain $X$
(a query) and returns the value of the function in $d$, $f(d)\in R$. The learner goal
is to find $f$ with a minimum number of
queries, optimal time complexity, and optimal resources.

In this survey, we present some of the results known from the
literature, different techniques used, some new problems, and open problems.
\end{abstract}

\tableofcontents

\newpage

\section{Introduction}\label{Int}
Robert Dorfman's paper in 1943 introduced the field of Group Testing.
The motivation arose during the Second World War
when the United States Public Health Service and the Selective service embarked upon a large scale project.
The objective was to weed out all syphilitic men called up for induction.
However, syphilis testing back then was expensive and testing every soldier
individually would have been very cost heavy and inefficient. A basic breakdown of a test is:
Draw sample from a given individual,
perform required tests and
determine the presence or absence of syphilis.
Suppose we have $n$ soldiers. Then this method of testing leads to $n$ tests.
Our goal is to achieve effective testing in a
scenario where it does not make sense to test $100,000$ people to get (say) $10$ positives.
The feasibility of a more effective testing scheme hinges on the following property.
We can combine blood samples and test a combined sample together to
check if at least one soldier has syphilis~\cite{GT}.

Let $S$ be the set of the $n$ soldiers and let $I\subseteq S$ be the set of the sick soldiers.
Suppose we know that the number of sick soldiers, $|I|$, is bounded by some integer $d$.
If~$T$ is the set of soldiers for which their blood samples is combined, then
the test is positive if and only if $I\cap T$ is not empty. Thus, we can regard the set of sick soldiers
$I$ as a Boolean function $f_I:2^S\to \{0,1\}$ and the answer of the test ``Is $I\cap T$ is not empty''
as $f_I(T)=1$ if and only if
$I\cap T\not=\O$. The goal is to identify the function $f_I$ (and therefore the sick soldiers)
from a minimal number of substitutions (tests) and optimal time.
We can also identify the set of soldiers with the set $[n]:=\{1,2,\ldots,n\}$
and regard each test as an assignment $a\in \{0,1\}^n$, where
$a_i=1$ if and only if the $i$th soldier blood is in the test.
Then the set $S=\{0,1\}^n$ is the set of all possible tests. The set of sick soldiers $I\subseteq [n]$
corresponds to a Boolean function $f'_I:S\to \{0,1\}$ where
$f'_I(x_1,\ldots,x_n)=\bigvee_{i\in I}x_i$ and $\vee$ is the Boolean or (disjunction). So this problem is also
equivalent to the problem of identifying, a hidden
Boolean conjunction of up to $d$ variables, with a minimal number of substitutions and optimal time.

Another interesting problem is the problem of learning decision tree with a minimal number of queries.
Let's say one has a restaurant and she wants to learn each customer tastes preference in food.
For every customer, she offers a sample of a meal that was never ordered
by the customer before and then receives some feedback. The customer tastes preference depends on some attributes of the food.
For example, ``sweet'', ``sour'', ``salty'', ``umami'', ``bitter'', ``greasy'', ``hot'' etc. Those are the attributes.
The goal is to learn (find out) the customer tastes preference from a minimal number of samples.
Each sample can be regarded as a set of attributes.
The customer tastes preference is the objective function. This function depends on the attributes,
and the value of the function is the customer feedback. In many cases, the target function
can be described as a decision tree. See the example in Figure~\ref{Tree01}.

\begin{figure}
\centering
\includegraphics[trim = 0 1cm 0 1cm,width=0.9\textwidth]{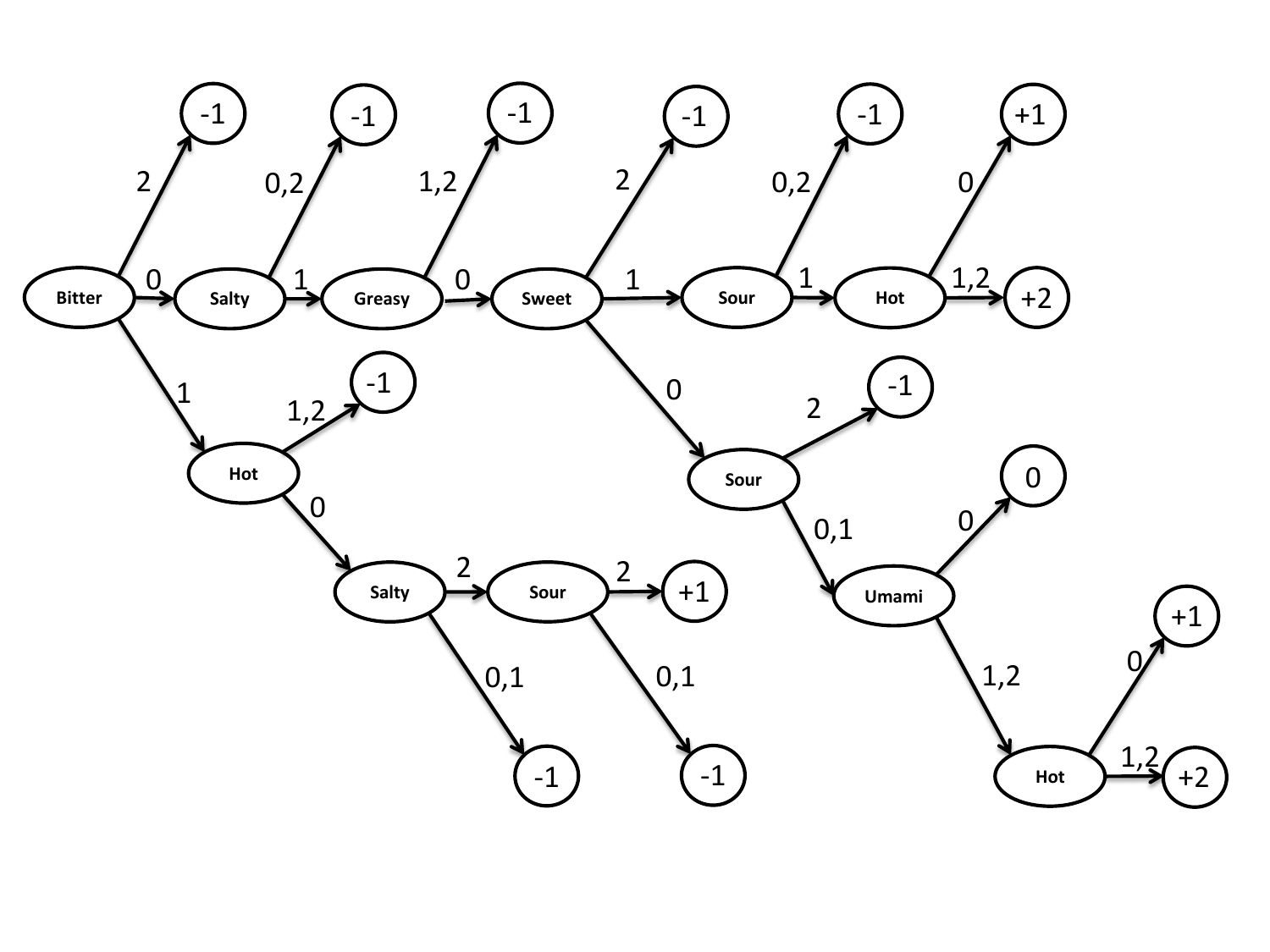}
\caption{A decision tree of the tastes preference of some customer. On the edges, $2,1,0$ is the degree
of flavor, ``very'', ``little'' and ``no'', respectively. In the leaves, $-1$ stands for ``not delicious'', $0$ for ``OK'', $+1$ for ``delicious'
and $+2$ for ``very delicious''.}
\label{Tree01}
\end{figure}

In the following subsection, we give a framework to the above problems
and many other similar problems.

\subsection{The Learning Model}
Let the {\it domain} ({\it instance space}) be the set $X_n\in\{X_j\}_{j\ge 1}$ and the {\it range} be the set $R_n\in\{R_j\}_{j\ge 1}$. Let $C_n$ be a {\it class of
representations of functions} $f:X_n\to R_n$ ({\it target class,
concept class}). Given a {\it teacher} ({\it black box, opponent player, responder})
that holds a ({\it target}) function ({\it concept}) $f$ from the class $C_n$. The {\it learner}
({\it player, questioner}) can ask the teacher {\it membership queries {\rm (for
Boolean functions. i.e. $R_n=\{0,1\}$)} {\rm or} substitution queries {\rm (for non-Boolean
functions)}}, i.e., it can send the teacher an element $d$ of the
domain $X_n$ and the teacher returns $f(d)$. The learner knows
$\{C_j,X_j,R_j\}_{j\ge 1}$. Our (the {\it learner}) ultimate goal is to
write an {\it (exact) learning algorithm} that learns $C=\cup_{j\ge
1}C_j$
with a minimum number of queries and optimal resources. That
is,
\begin{enumerate}
\item {\bf Input:} The learning algorithm receives the input $n$
and has access to an {\it oracle} MQ$_f$ that answers membership/substitution queries for
the target function~$f\in C_n$.

\item {\bf Query complexity:} It asks the teacher a minimum number
of membership/substitution queries.

\item {\bf Exact learning:} It either learns (finds, outputs) $g\in
C_{n}$ such that $g$ is logically equivalent to $f$, $g= f$, ({\it proper learning}) or
learns $h\in H_n\supseteq C_n$
such that $h= f$ ({\it non-proper learning from $H_n$}).

\item  {\bf Resources Complexity:} It runs in linear/polynomial/optimal time
complexity, optimal space complexity, an optimal number of random bits or/and other optimal resources.
\end{enumerate}

\noindent
The following decision problems are also considered in the literature
\begin{enumerate}
\item {\bf Equivalent test}: Given two teachers that have
two functions from $C_n$ each. Test whether the two functions
are equivalent.

\item {\bf Identity test from $H_n$}: Given a teacher that has
a function $f$ from $C_n$.
Given a function $h\in H_n$.
Test whether~$f= h$.

\item
{\bf Zero test}: Given a teacher that has
a function $f$ from $C_n$. Test whether~$f=~0$.
\end{enumerate}

The number of queries ({\it query complexity}) and the resources complexities are expressed as
functions in $n$ and some other parameters that depend on the class being learned.
In the literature, there are many other variations of the above problems, and we will mention some of them in this survey.

This problem has different names in different areas: Conditional
and unconditional Tests~\cite{M83}, Combinatorial Search~\cite{K73}, Interpolation
\cite{BT88}, Combinatorial Group Testing~\cite{DH00},
Exact Learning from Membership Queries \cite{A87}, Inferring \cite{GS92},
Identifying \cite{GKS93}, Test Recognition~\cite{GDPK06}, Active Learning \cite{S10},
Reconstruction~\cite{KS09b} and Guessing
Game \cite{W1}.
The decision problems are also called Testing, Functional
Verification, Teaching, Hitting Set, and when $f$ is polynomial, it is called
Black Box polynomial identity testing (PIT) \cite{S09,SY10}.

There are many other learning models, but, throughout this survey,
when we say exact learning or learning we mean exact learning from
membership queries or substitution queries only.

In this survey, we present some of the results known from the
literature, different techniques used and some open problems.

\subsection{Domain and Range}

Throughout this survey, we will omit the subscript $n$ from
$C_n,X_n$ and $R_n$. In principle, the domain $X$ and the range
$R$ can be any two sets, but since mathematical models can
explain many natural phenomena, most of the sets considered in
the literature are either finite or have some algebraic structure such as rings,
fields, integers and real numbers.

Therefore, the domains and ranges
considered in the literature are: The {\it Boolean set}
that can be either $\{0,1\}$, $\{-1,+1\}$, $\{+,-\}$  or
the binary field $F_2$. The {\it finite discrete set} can be any finite set or
a finite set with some algebraic structure such as the ring $Z_n$ of
integers modulo $n$, or the finite field $F_{q}$ with $q$ elements
($q$ is a power of prime). The {\it infinite discrete set} can be
any countably infinite set such as the set of integers $Z$ or the set of rational numbers $Q$. The
{\it infinite set} (uncountable)
can be any set with some algebraic structure such as the real numbers $\Re$ or the complex
numbers ${\cal C}$. Also, the cartesian product of any finite number
of the above sets is considered in the literature.

\subsection{Classes of Functions}\label{Def}
In this section, we will list the most studied classes in the literature,
in different fields of computer science.

{\bf Boolean Function Classes:} When the range of the function is $R_n=\{0,1\}$
we call the function {\it Boolean function}.
Here we will consider classes $C$ of Boolean functions when the
domain is $X_n=\{0,1\}^n$. For any
class defined below when we say that
$f$ is $C$, we mean that $f\in C$.
Abusing the terminology, every function $f\in C$ is
regarded as a representation of the function (formula)
and as a function, and we will use both interchangeably.

The most studied classes in the literature are:

\begin{enumerate}\setlength\itemsep{.5em}
\item {\bf Variable (Var)}: The class Var is the class of functions $\{x_1,\ldots,x_n\}$,
where for $a\in \{0,1\}^n$, $x_i(a)=a_i$. We also define Lit$=\{x_1,\ldots,x_n\}
\cup \{\bar x_1,\ldots,\bar x_n\}$ the class of {\it literals}.
Here $\bar x$ is the logic negation of $x$.

Learning the class Var is equivalent to playing the R\'{e}nyi-Ulam game, \cite{P02,R61,U76}.

\item {\bf $d$-Monotone Clause ($d$-MClause)} and {\bf MClause}: The
class $d$-MClause is the class of all functions
$f_S:\{0,1\}^n\to \{0,1\}$ where $S\subseteq [n]:=\{1,2,\ldots,n\}$
and $|S|\le d$ such that $f_S(x_1,\ldots,x_n)=1$ if and only if
$x_i=1$ for some $i\in S$. When $S=\emptyset$ then $f_\emptyset= 0$. Such function can also be expressed as
a logic {\it monotone clause} $f_S=x_{i_1}\vee \cdots \vee x_{i_k}$ where $S=\{i_1,\ldots,i_k\}$, $k\le d$ and $\vee$ is the logic ``or'' function (disjunction).
We denote $n$-MClause by {\bf MClause}.

Learning $d$-MClause is equivalent to group testing, \cite{D43,DH00,DH06}. See
many other equivalent problems in \cite{PR11} and reference within.

\item {\bf $d$-Clause} and {\bf Clause}:
The class $d$-Clause is the class of all functions
$f_{S,R}:\{0,1\}^n\to \{0,1\}$ where $S\cap R=\emptyset$, $S\cup R\subseteq [n]$
and $|S\cup R|\le d$ such that $f_{S,R}(x_1,\ldots,x_n)=1$ if and only if
$x_i=1$ for some $i\in S$ or $x_j=0$ for some $j\in R$.
Such function can be expressed as
a logic {\it clause}
$f_{S,R}=x_{i_1}\vee \cdots \vee x_{i_k}\vee \bar x_{j_1}\vee \cdots \vee \bar x_{j_r}$ where $S=\{i_1,\ldots,i_k\}$, $R=\{j_1,\ldots,j_r\}$,
and $r+k\le d$.
We denote $n$-Clause by {\bf Clause}.

\item \label{Doo} {\bf $d$-Monotone Term ($d$-MTerm)},
{\bf $d$-Term ($d$-Term)},
{\bf MTerm} and {\bf Term}:
The same as the above classes, but replace $\vee$ with the logic ``and'' function $\wedge$
(Conjunction).
The functions in MTerm are sometimes called {\it monomials},
and the class MTerm is also denoted by {\bf Monomial}. That is, a monomial is a conjunction of variables, i.e., $x_{j_1}\wedge x_{j_2}\wedge \cdots \wedge x_{j_r}$
where $1\le j_1<j_2<\cdots<j_r\le n$. Here we will sometimes use the
arithmetic $\times$ of the field $F_2$ for $\wedge$ and write
$x_{j_1}\wedge x_{j_2}\wedge \cdots \wedge x_{j_r}$ as $x_{j_1} x_{j_2} \cdots  x_{j_r}$.

For a class $C$, the dual class of $C$ is the class
$$C^D=\left\{\left.\overline{f(\bar x_1,\ldots,\bar x_n)}\ \right| \ f\in C\right\}.$$ Obviously, $(C^D)^D=C$, $d$-Clause$^D$=$d$-Term and $d$-MClause$^D$=$d$-MTerm.

\item {\bf $d$-XOR} and {\bf XOR}:
The same as the $d$-Term class, but replace $\wedge$ with the logic exclusive or function $\oplus$.
Here, we will instead use the arithmetic $+$ of the finite field $F_2=\{0,1\}$. Since $\bar x=x+1$, every function
in XOR is of the form $f=x_{i_1}+\cdots+x_{i_k}+\xi$ where $1\le i_1<i_2<\cdots<i_k\le n$ and $\xi\in\{0,1\}$.

\item {\bf $d$-Junta}: Let $f:\{0,1\}^n\to \{0,1\}$.
A variable $x_i$ is said to be {\it relevant} in $f$ if
there are two assignments $a,b\in \{0,1\}^n$ such that $a_i\not=b_i$, for all $j\not=i$ we have $a_j=b_j$,
and $f(a)\not= f(b)$.
The class $d$-Junta is the class
of all Boolean functions with at most $d$ relevant variable.
This function can be represented by a truth table of size $2^d$
of all the relevant variables.

\item\label{MF}  {\bf $d$-MJunta}: For two assignments $a,b\in \{0,1\}^n$ we write $a\le b$ if
for every $i$, $a_i\le b_i$. A Boolean function $f:\{0,1\}^n\to\{0,1\}$ is {\it monotone} if
for every two assignments $a,b\in\{0,1\}^n$, if $a\le b$ then $f(a)\le f(b)$.
It is easy to see that Monotone functions are closed under disjunction and conjunction.
That is, if $f$ and $g$ are monotone functions then $f\wedge g$ and $f\vee g$ are monotone functions.

The class $d$-MJunta is the class of all monotone functions in $d$-Junta.
That is, the class of all monotone functions with at most $d$ relevant variables.

\item {\bf Decision Tree (DT)}: One of the important
representations of Boolean functions $f:\{0,1\}^n\to \{0,1\}$ is
decision tree. A {\it decision tree} formula is defined as follows: The
constant functions $0$ and $1$ are decision trees. If $f_0$ and
$f_1$ are decision trees then, for all $i$, $$\mbox{``$f'=$[if $x_i=0$ then $f_0$
else $f_1$]"}$$ is a decision tree (can also be expressed as $f'=x_i
f_1\vee\bar x_i f_0$ or $f'=x_i
f_1+\bar x_i f_0$). Every decision tree $f'$ can be represented as
a tree $T(f')$. If $f'= 1$ or $0$ then $T(f')$ is a node
labeled with $1$ or $0$, respectively. If $f'=$[if $x_i=0$
then $f_0$ else $f_1$], then $T(f')$ has a root labeled with $x_i$
and has two outgoing edges. The first edge is labeled with $0$ and
is pointing to the root of $T(f_0)$ and the second is labeled with
$1$ and is pointing to the root of $T(f_1)$. See Figure~\ref{Tree02}.

The {\it depth}
of the decision tree $f'$ is the depth of the tree $T(f')$.
That is the number of edges of the longest path from the root to a leaf in a tree.
The {\it size} of the
decision tree $f'$ is the number of leaves in $T(f')$, that is,
the number nodes in $T(f')$ that are labeled with $0$ and $1$.

Every Boolean function $f:\{0,1\}^n\to \{0,1\}$ can be represented as a DT.
The representation is not unique. The following are subclasses of DT.

\begin{enumerate}\setlength\itemsep{0em}
\item {\bf Depth $d$ Size $s$ Decision Tree ({DT}$_{d,s}$)}: The
class $\DT_{d,s}$ is the class of all decision trees of depth
at most $d$ and size at most $s$.

\item {\bf Depth $d$ Decision Tree (DT${}_d$)}:  The
class $\DT_{d}$ is the class of
all decision trees of depth at most~$d$. That is, $\DT_d=\DT_{d,2^d}$.

\item {\bf Monotone DT ({MDT}$_{d,s}$, MDT${}_d$)}: functions in the above classes
that are monotone.

\item {\bf Decision List (DL)},\cite{R87}: functions $f\in$DT where every internal node
in $T(f)$ is pointing to at least one leaf.

\item {\bf Depth $d$-Decision List ($d$-DL)}: $d$-DL is a decision list
of depth at most $d$.
\end{enumerate}

Learning decision tree is equivalent to solving problems in databases, decision table programming, concrete complexity theory, switching theory, pattern recognition, and taxonomy,~\cite{M82},
computer vision,~\cite{AMMRS93}.

\begin{figure}
\centering
\includegraphics[trim = 0 1cm 0 1cm,width=0.9\textwidth]{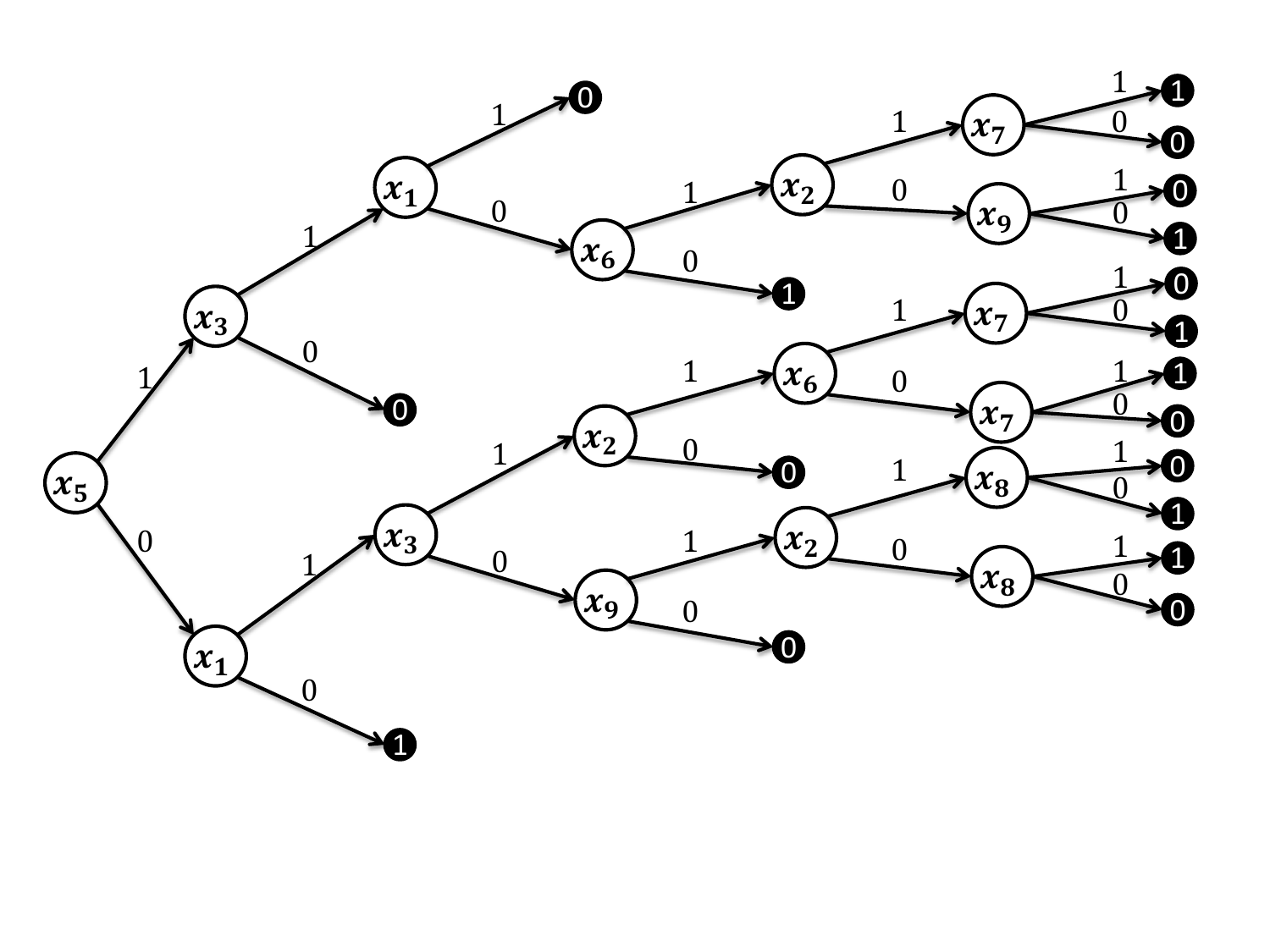}
\caption{A decision tree of depth $6$ and size 18.}
\label{Tree02}
\end{figure}

\item  {\bf Disjunctive Normal Form (DNF)}: A DNF is another
important representation of Boolean function
$f:\{0,1\}^n\to \{0,1\}$. A DNF formula
is a formula of the form $$f=\bigvee_{i=1}^s T_i$$ where each $T_i\in$Term is a term.
The {\it size} of $f$ is $s$.

Every Boolean function $f:\{0,1\}^n\to \{0,1\}$ can be represented as a DNF.
The representation is not unique. It is easy to see that every decision tree of size $s$ can be
represented as DNF of size at most $s$.

The subclasses of DNF considered in the literature are
\begin{enumerate}\setlength\itemsep{0em}
\item {\bf $r$-DNF}: The class of DNFs with terms from $r$-Term.
\item {\bf $s$-term DNF}: The class of DNFs with at most $s$ terms.
\item {\bf $s$-term $r$-DNF}: The class of DNFs with at most $s$ terms each of which is an $r$-Term.
\item {\bf Read-Once $C$:} Here $C$ is one of the above classes.
Read-Once $C$ is
the class of functions $f$ in $C$ where each variable appears at most
once in $f$.
\item {\bf Read-Twice, Read-Thrice, Read-$t$ $C$:} The class of functions
$f$ in $C$ where each variable appears at most
twice (resp. three times and $t$ times) in~$f$.
\end{enumerate}

\item {\bf Monotone DNF (MDNF)}: The class MDNF
 is the class of DNF with monotone terms (i.e., terms in MTerm).
Every monotone function (See the definition in item~\ref{MF})
has a monotone DNF representation. This representation is one of the most popular
canonical structures for representing Boolean functions. If $f = M_1\vee M_2 \vee \cdots \vee M_s$ where
each $M_i$ is a monomial and no two monomials $M_i$ and $M_j$, $i\not=j$ satisfies $M_i\wedge M_j= M_i$,
then we say that $f$ is a {\it reduced monotone DNF}.
Every monotone Boolean function $f$ has a unique representation
as a reduced monotone DNF \cite{A88}. This representation is uniquely
determined by the {\it minterms} of the function. That is, the
assignments $a\in\{0,1\}^n$ where $f(a)=1$ and flipping any entry
that is $1$ in $a$ to $0$ changes the value of the function to
zero. Each minterm $a$ of $f$ corresponds, one-to-one, to a monomial $M=\vee_{a_i=1}x_i$
in the reduced monotone DNF representation of $f$. The following are subclasses of MDNF.
\begin{enumerate}\setlength\itemsep{0em}
\item {\bf $r$-MDNF}: The class of MDNFs with monomials of size at most $r$. That is, terms from $r$-MTerm.
\item {\bf MDNF}: The class MDNF is $n$-MDNF.
\item {\bf $s$-term MDNF}: The class of MDNFs with at most $s$ monomials.
\item {\bf $s$-term $r$-MDNF}: The class MDNFs with at most $s$ monomials
    of size at most $r$.
\item {\bf Read-Once, Read-Twice, Read-Thrice, Read-$t$ $C$}, where $C$ is one of the above classes, is
the class of functions $f$ in $C$ where each variable appears at most
once (resp. twice, three times and $t$ times) in $f$.
\end{enumerate}

Learning Monotone DNF and subclasses of Monotone DNF equivalent to problems
in computational biology that arises in whole-genome shotgun sequencing,~\cite{ABKRS04},
and DNA phisical mapping,~\cite{GK98}.

\item  {\bf Conjunctive Normal Form (CNF)}: The class CNF is the dual class (See the definition in item~\ref{Doo}) of
DNF (where $\wedge$ is replaced with $\vee$ and vice versa). In a similar way
as above we define the classes {\bf $r$-MCNF}, {\bf MCNF}, {\bf $s$-clause CNF},
 {\bf $s$-clause MCNF}, {\bf $s$-clause $r$-CNF} and {\bf $s$-clause $r$-MCNF}.

\item {\bf CDNF}. The class of CDNF is the class of formulas of
the form $(f,g)$ where $f$ is a DNF,  $g$ is a CNF and $f= g$.
The {\it size} of $(f,g)$ as $s+t$ where $f$ is $s$-term DNF and $g$ is $t$-clause CNF.

The following are subclasses of CDNF
\begin{enumerate}\setlength\itemsep{0em}
\item {\bf CDNF$_{s,t}$}. The class of CDNF$_{s,t}$ is the class CDNF, $(f,g)$, where $f$ is a DNF of size at most $s$ and  $g$ is a CNF of size
at most $t$.
\item{\bf $r$-CDNF$_{s,t}$}: The class of $(f,g)\in$CDNF$_{s,t}$ where $f$ is $r$-DNF
of size at most $s$ and $g$ is $r$-CNF of size at most $t$.
\item{\bf $r$-CDNF}: The class $\cup_{s,t}$CDNF$_{s,t}$.
\item {\bf MCDNF}: The class of Monotone CDNF.
\item {\bf MCDNF$_{s,t}$}: The class of Monotone CDNF$_{s,t}$.
\item{\bf $r$-MCDNF}: The class of Monotone $r$-CDNF.
\end{enumerate}

Learning CDNF is equivalent to problems in data-mining, graph theory and reasoning and knowledge representation,~\cite{DMP99}.

\item {\bf Boolean Multivariate Polynomial (BMP)}. The class BMP is
the class of multivariate polynomials over the binary field $F_2$. That is, a function
$f:F_2^n\to F_2$ of the form $$f=M_1+M_2+\cdots+M_s$$ where
each $M_i$ is a monomial. The {\it size} of $f$ is $s$.

Every Boolean function $f:F_2^n\to F_2$ can be represented as a BMP.
The representation is unique.
It is easy to see that every decision tree of size $s$ and depth $t$ can be
represented as BMP of size at most $2^ts$.

\begin{enumerate}\setlength\itemsep{0em}
\item {\bf $r$-BMP}: The class of BMPs with monomials of size at most $r$, i.e., in $r$-MTerm.
This class is also called the class of {\it multivariate polynomial of degree $r$ over $F_2$}.
\item {\bf $s$-monomial BMP}: The class BMPs with at most $s$ monomials.
This class is also called the class of {\it sparse multivariate polynomial over $F_2$}.
\item {\bf $s$-monomial $r$-BMP}: The class of BMPs with at most $s$ monomials
of size at most $r$. This class is also called the class of {\it sparse multivariate polynomial of degree $r$ over $F_2$}.
\end{enumerate}

\item {\bf XOR of Terms (XT)}: The class XT is the class of XOR of terms, $T_1+T_2+\cdots+T_s$
where $T_i\in$ Term.
\begin{enumerate}\setlength\itemsep{0em}
\item {\bf $r$-XT}: The class of XTs with terms of size at most $r$.
\item {\bf $s$-term XT}: The class of XTs with at most $s$ terms.
\item {\bf $s$-term $r$-XT}: The class of XTs with at most $s$ terms of size at most $r$.
\end{enumerate}
Notice that XT with terms from MTerm is BMP. Since every term of size $r$ can be
represented as $2^r$-monomial $r$-BMP, every $s$-term $r$-XT is $(2^rs)$-monomial $r$-BMP.

\item {\bf Deterministic Finite Automaton (DFA)},\cite{MP43}:
A DFA is a $5$-tuple
$A=(Q,\Sigma,\delta,q_0,F)$ that can be also represented
as a directed graph $G=(V,E)$ with labeled edges where $V=Q$ is a finite set of {\it states} (the vertices), and $q_0\in Q$ is the {\it start state}.
$\Sigma$ is a finite set of {\it symbols} called the {\it alphabet}.
$\delta$ is the {\it transition function} $\delta:Q\times \Sigma\to Q$.
The edge $(v,u)\in E$ in $G$ is labeled with $\sigma\in \Sigma$ if and only if $\delta(v,\sigma)=u$. This transition function defines, for every {\it string} $s\in \cup_{i\ge 0}\Sigma^i$, a unique path in the graph $q_0,q_1,\ldots, q_{|s|}$ (here, $|s|$ is the number of symbols in $s$) that starts
from $q_0$ and for every $0\le i\le |s|-1$, $\delta(q_{i-1},s_i)=q_{i}$. We denote the final state in this path $q_{|s|}$ as $\delta(q_0,s)$.
The set $F\subset Q$ is the set of {\it accept states}.

Every DFA $A$ defines
a Boolean function $f:\cup_{i\ge 0}\Sigma^i\to \{0,1\}$ where $f(s)=1$ if and only if
$\delta(q_0,s)\in F$. When $\Sigma=\{0,1\}$ then a {\it DFA for the Boolean function}
$f:\{0,1\}^n\to \{0,1\}$ is a DFA such that: for every $a\in\{0,1\}^n$
we have $f(a)=1$ if and only if $\delta(q_0,a)\in F$.

\item {\bf Boolean Multiplicity Automata Function (BMAF)},\cite{S61}: A Boolean Multiplicity
Automata Function is a function of the
form:
$$f(x_1,\ldots,x_n)=A_1(x_1)A_2(x_2)\cdots A_n(x_n)$$ where each
$A_i(x_i)$ is $s_i\times s_{i+1}$ matrix that its entries are
Boolean univariate polynomials in $x_i$ over $F_2$, i.e., $ax_i+b$ for $a,b\in F_2$, and $s_1=s_{n+1}=1$. The {\it size} of a BMAF is defined as $\max_i s_i$.

See~\cite{BBBKV00} for other ways to represent this class.

\item {\bf Boolean Halfspace} ({\bf Perceptron, Threshold}) {\bf (BHS)}: A Boolean
Halfspace is a function $f:\{0,1\}^n\to \{0,1\}$ of the following
form:
$$f(x_1,\ldots,x_n)=\left\{ \begin{array}{ll}
1 & \mbox{if}\ w_1x_1+w_2x_2+\cdots+w_nx_n\ge u\\
0 & \mbox{otherwise}
\end{array}\right.$$ where $w_1,\ldots,w_n,u$ are real numbers.
The constants $w_1,\ldots,w_n$ are called
the {\it weights} of the Halfspace, and $u$
is called the {\it threshold}. For $W\subseteq \Re$ we define
\begin{enumerate}\setlength\itemsep{0em}
\item {\bf BHS$(W)$}: The class of Boolean Halfspaces with weights
$w_i\in W$.
\item {\bf $d$-BHS$(W)$}: The class of functions in BHS$(W)$ with at most
$d$ relevant variables.
\end{enumerate}

\item {\bf Boolean Circuit (BC) and Boolean Formula (BF)} A
Boolean circuit over the set of variables
$x_1,\ldots,x_n$ is a directed acyclic graph where every node in it
with indegree zero is called an input gate and is labeled by
either a variable $x_i$ or a Boolean constant $\{0,1\}$. Every other gate is either a node with indegree one and is
labeled $\neg$ (unary NOT) or a node with indegree two
and is labeled by either, $\wedge$ (binary AND) or $\vee$ (binary OR). A Boolean formula is a
circuit in which every gate has outdegree one.

The {\it size} of a Boolean circuit is the number of gates in it, and its
{\it depth} is the length of the longest directed path in it.
\begin{enumerate}\setlength\itemsep{0em}
\item {\bf Monotone Boolean Circuit (MBC) and Monotone Boolean Formula (MBF)}
MBC and MBF are Boolean circuit and Boolean formula, respectively, with no $\neg$ gate.
\item {\bf Read Once Formula (ROF)}. The class of Boolean read-once formula.
A Boolean read-once formula is a formula such that
every input variable $x_i$ appears in at most one input gate.
\item {\bf Monotone Read Once Formula (MROF)}. The class of monotone read-once formula.
\item {\bf Read-Once, Read-Twice, Read-Thrice, Read-$t$ $C$}, where $C$ is one of the above classes, is
the class of functions $f$ in $C$ where each variable appears at most
once (resp. twice, three times and $t$ times) in $f$.
\end{enumerate}
\end{enumerate}

See other classes in \cite{A88,A87,AFP92,AP91,B95,B97,BHH95a,BM95,CGHT11,CGL97,D96,DMP99,FGMP96,H95,SS96}.

Here are relations between some of the classes mentioned above.
$$ $$
\begin{figure}
\centering
\includegraphics[trim = 0 5cm 0 5cm,width=0.8\textwidth]{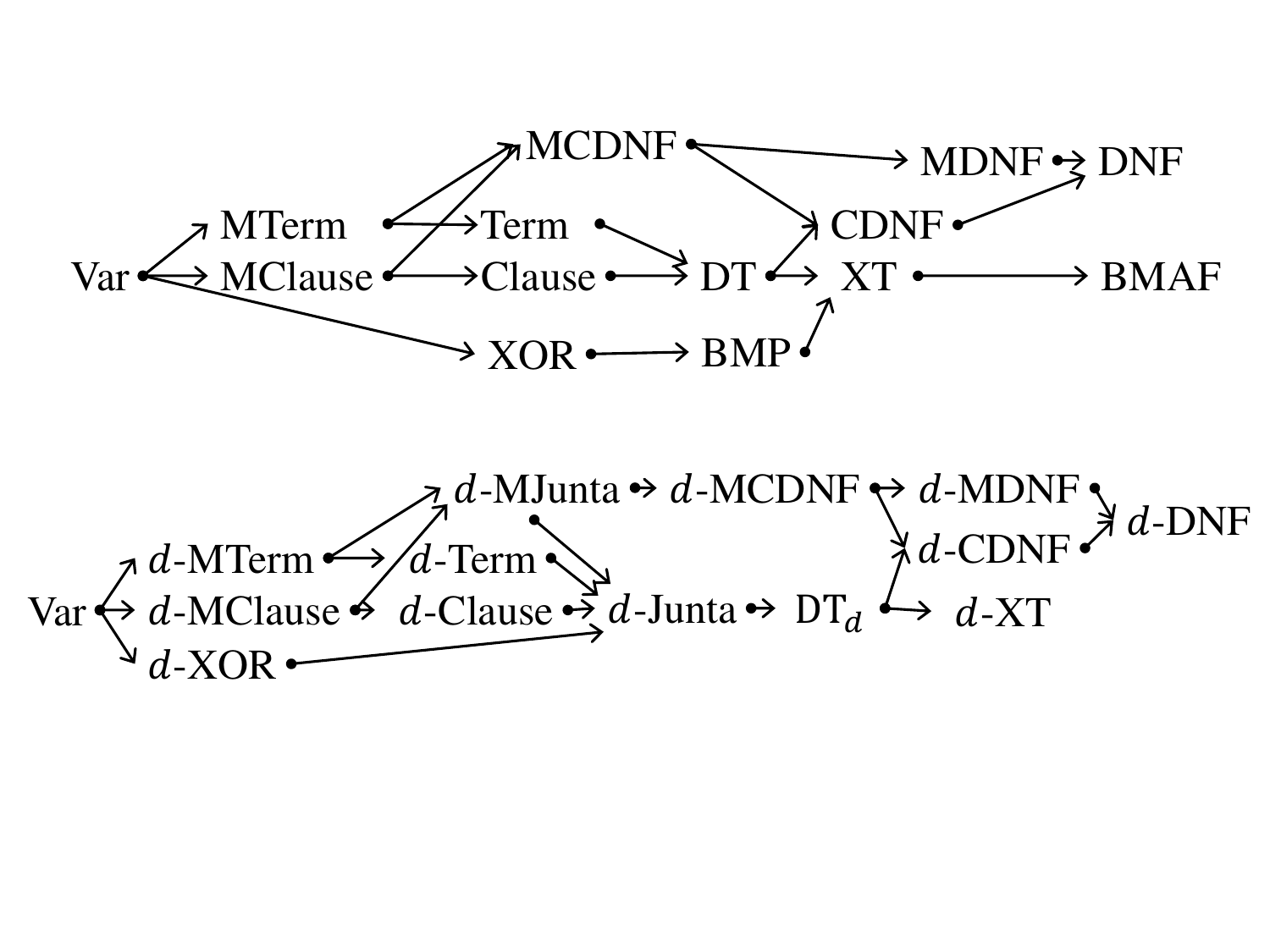}
\caption{Diagram of the classes}
\label{Diagram01}
\end{figure}

\ignore{
{\small
$$
\begin{array}{llllclccclllll}
&&&&&&&{\rm MCDNF}  &\rightarrow &{\rm MDNF} & \rightarrow & {\rm DNF}&&                \\
[-.23cm]
&&&&&&&   &\searrow&&{\nearrow}&&&\\[-.23cm]
&&&{\rm MTerm}&\rightarrow&{\rm Term}&&&  &{\rm CDNF}  & & &&\\
[-.23cm]
&\nearrow&&&&&\searrow&   &{\nearrow}&&&&&\\
[-.20cm]
{\rm Var}&\rightarrow&&{\rm MClause}&\rightarrow&{\rm Clause}&\rightarrow&{\rm DT}&\rightarrow &{\rm XT}&\rightarrow&{\rm BMAF}&&\\
[-.23cm]
&\searrow&&&&&&   &{\nearrow}&&\nearrow &&&\\[-.23cm]
&&&{\rm XOR}&&&& {\rm XOR}  &\rightarrow&{\rm BMP}&&&&\\
\end{array}
$$

$$
\begin{array}{llllclclrlllll}
&&&&&&&  d\mbox{-MJunta}\rightarrow\subset d\mbox{-MCDNF}&\rightarrow &d\mbox{-MDNF} & \rightarrow & d\mbox{-DNF}&&                \\
[-.23cm]
&&&&&&&   &\searrow&&{\nearrow}&&&\\[-.23cm]
&&&d\mbox{-MTerm}&\rightarrow&d\mbox{-Term}&&\ \ \ \ \ \
\downarrow&  &d\mbox{-CDNF}  & & &&\\
[-.23cm]
&\nearrow&&&&&\searrow&   &{\nearrow}&&&&&\\
[-.20cm]
{\rm Var}&\rightarrow&&d\mbox{-MClause}&\rightarrow&d\mbox{-Clause}&
\rightarrow& \subset {d\mbox{-Junta}}\rightarrow\subset{\rm DT}_d\longrightarrow\cdot\cdot&\rightarrow &d\mbox{-XT}&&&&\\
[-.23cm]
&\searrow&&&&&{\nearrow}&   &&&&&&\\[-.23cm]
&&&d\mbox{-XOR}&&d\mbox{-XOR}&& &&&&&&\\
\end{array}
$$}}

For two classes $C_1$ and $C_2$ we write $C_1\subseteq C_2$
(written as $C_1\to C_2$ in the above diagram) if every function in $C_1$ of size $s$ is equivalent to
a function in $C_2$ of size $O(s)$.

As for functions that are not Boolean, the literature
is poor in studying the exact learnability of classes of functions
with finite discrete domain or/and range from membership queries only.
On the other hand, there is a substantial body of literature on learning
and testing arithmetic classes.

\bigskip
We now give some of the arithmetic classes defined in the literature

\bigskip
\noindent
{\bf Arithmetic Classes:} Arithmetic classes represent function $f:X\to R$
where $R$ is an algebraic structure such as field or ring.
For exact learning, the most investigated arithmetic classes in the literature
are

\begin{enumerate}\setlength\itemsep{.5em}

\item {\bf $(r,V)$-Linear Functions ($(r,V)$-LF)}, where $r$
 is an integer, $V\subset \Re$ and $\Re$ is the set of real numbers. An $(r,V)$-LF is a function $f:\{0,1\}^n\to \Re$ of the form $v_1x_{i_1}+\cdots+v_{r'}x_{i_{r'}}$ where $r'\le r$
 and $v_i\in V$ for all $i=1,\ldots,r'$. The class {\bf $r$-LF} is the class $(r,\{0,1\})$-LF
 and {\bf LF} is the class $n$-LF.

Learning $(r,V)$-LF is equivalent to coin weighing problem~\cite{B12b} and signature coding problem~\cite{BG07}.

\item {\bf $(r,V)$-Quadratic Functions ($(r,V)$-QF)}, where $r$
 is an integer and $V\subset \Re$. A $(r,V)$-QF is a function
$f:\{0,1\}^n\to \Re$ of the form $x^TAx$ where $x\in \{0,1\}^n$
and $A$ is a symmetric $n\times n$ matrix with at most $r$ non-zero entries from~$V$.
The class {\bf $r$-QF} is the class $(r,\{0,1\})$-QF.

Learning $(r,V)$-QF is equivalent to problems in molecular biology~\cite{BGK05}.

\item {\bf Multivariate Polynomial (MP):} Let $F$ be a field. A multivariate
polynomial over $F$ is a function $f:F^n\to F$ of the form $$f=\sum_{i\in
I} a_ix_1^{i_1}\cdots x_n^{i_n}$$ where $I\subseteq N^n$,
$N=\{0,1,2,\cdots\}$ and $a_i\in F$. The {\it size} of $f$ is
$|f|:=|I|$. The term $x_1^{i_1}\cdots x_n^{i_n}$ is called {\it
monomial}. The monomial is called {\it $t$-monomial} if $|\{j\ |\
i_j\not=0\}|\le t$. The multivariate polynomial is said to be of
{\it degree} $d$ if $i_1+\cdots+i_n\le d$ for all $i\in I$,
$s$-{\it sparse} if $|I|\le s$ and {\it with $t$-monomials} if all
its monomials are $t$-monomials.

When the field $F$ is finite then every function $f:F^n\to F$ can be
represented as a multivariate polynomial. This fact is not true for
infinite fields.

\item {\bf Multiplicity Automata Function}: A Multiplicity
Automata Function (MAF) over the field $F$ is a function of the
form
$$f(x_1,\ldots,x_n)=A_1(x_1)A_2(x_2)\cdots A_n(x_n)$$ where each
$A_i(x_i)$ is $s_i\times s_{i+1}$ matrix that its entries are
linear functions in $(x_1,\ldots,x_n)$ (i.e., $\sum_ia_ix_i+b$ where $a_i,b\in\Re$) and $s_1=s_{n+1}=1$.
The size of a MAF $f$ is $\max_i s_i$.

This class contains the class MP in a sense that every MP of size
$s$ has a MAF of size $s$.

See \cite{BBBKV00} for other representations of MAF.

\item {\bf Arithmetic Circuit (AC) and Arithmetic Formula (AF)} An
arithmetic circuit over the field $F$ and the set of variables
$x_1,\ldots,x_n$ is a directed acyclic graph where every node in it
with indegree zero is called an input gate and is labeled by
either a variable $x_i$ or a field element. Every other gate is
labeled by either $+$ or $\times$, in the first case, it is a sum
gate and in the second a product gate. An arithmetic formula is a
circuit in which every gate has outdegree one.

The {\it size} of a circuit is the number of gates in it, and its
{\it depth} is the length of the longest directed path in it. The
{\it degree} of a circuit is equal to the degree of the polynomial
output by the circuit.

\item {\bf Arithmetic Read-Once Formula (AROF)}.
An arithmetic read-once formula is a formula such that
every input variable $x_i$ appears in at most one input gate.

\end{enumerate}

Here are relations between some of the classes we've
defined

$$ \begin{array}{rcccl}
&&{\rm AF}&\rightarrow&{\rm AC}\\
[-.23cm] &\nearrow &&\nearrow&\\ [-.28cm]
 {\rm MP}&\rightarrow& {\rm MAF}&&
\end{array}
$$

See other classes in
\cite{BB98,S09,S14,SY10} and references therein.

\subsection{Learning Algorithms and Complexity}

The learning algorithm can be {\it sequential} or {\it parallel},
{\it deterministic} or {\it randomized} and {\it adaptive} (AD), {\it $r$-round}
($r$-RAD)  or {\it
non-adaptive} (NAD).

In the adaptive algorithm, the queries can depend on the answers to the previous ones. In the non-adaptive algorithm they are independent of the previous one and; therefore, one can ask all the queries in one parallel step.
We say that an adaptive algorithm is $r$-{\it round adaptive} ($r$-RAD) if
it runs in $r$ stages where each stage is non-adaptive.
That is, the queries may depend on the answers to the queries in the previous stages
but independent on the answers to the queries of the current stage.

The randomized algorithm can be either Monte Carlo (MC) or Las Vegas (LV).
A {\it Monte Carlo algorithm} is a randomized algorithm whose running time is deterministic, but whose output may be incorrect with probability at most~$\delta$.
A {\it Las Vegas algorithm} is a randomized algorithm that always gives a correct hypothesis. That is, it always produces a hypothesis that is equivalent to the target function. The complexity of a Las Vegas algorithm is measured by the expected running time, the expected number of queries
and the expected number of rounds.

The goal is to ask the minimum number of queries and minimize the
running time and space complexity of the algorithm and/or other
resources such as the number of processors (for parallel algorithms)
or the number of random bits (for randomized algorithms).

\subsection{Polynomially, Efficiently and Optimally Learnable}\label{VE}
In this subsection and the next, we try to unify the different definitions used in the literature
of the efficiency of the
query complexity and time complexity of exact learning algorithms.
We will use the following new terminologies defined below: ``learnable'', ``polynomially learnable''
``efficiently learnable'', ``almost optimally learnable''  and ``optimally learnable''.

Let $C$ be a class of functions. Let $\OPT_A(C)$ be the minimum number
of membership queries that a learner, with unlimited computational power,
needs to learn $C$ with algorithms of type $A$. The algorithm type, $A$,
can be adaptive (AD), non-adaptive (NAD) or $r$-round ($r$-RAD). For
example, we will use $\OPT_\A$ for the adaptive algorithm and $\OPT_\NA$ for the non-adaptive algorithm.
When the algorithm is randomized we also add,
as a subscript, MC for Monte Carlo algorithms and LV for Las Vegas algorithms.

In complexity theory, a polynomial time algorithm is an algorithm that runs in
polynomial time in the input size. In the exact learning model, the time complexity of learning
the class $C$ is, at
least, the query complexity, $\OPT_A(C)$, which can be exponential in
the target function size.  Therefore, {\it polynomial time learning
algorithm}
for
$C$ will be defined as a learning algorithm that asks
$poly(\OPT_A(C),$ $n)$ queries and runs in time $poly(\OPT_A(C),n)$,
where $n$ is the size of the elements in the domain $X$.
Such classes are called {\it polynomially learnable} or just {\it learnable} classes.
This is the definition used in the literature for learnability of classes.

Since the time complexity of any learning algorithm for $C$ is at least $n\cdot \OPT_A(C)$ we may say that
learning algorithms that run in time $poly(\OPT_A(C),$ $n)$ are ``efficient algorithms'' in time.
However, this is not true for the query complexity.
We will argue here,
by the following example, that the above
definition of $poly(\OPT_A(C),n)$ for the query complexity is not the best definition
for query-efficiency of exact learning from membership queries.

Take for example the class $C=d$-MClause. We will show
in Subsection~\ref{GTA} that $\OPTA(C)=\Theta(d\log n)$.
Therefore, one would expect that a query-efficient learning algorithm for this class
asks $poly(d,\log n)$ queries and not $poly(d\log n,n)=poly(n)$ queries as defined above.
The time complexity cannot be less than $n\cdot \OPT_A(C)$, so the definition of $poly(\OPT_A(C),n)$
in the time complexity is passable.

Therefore, we will suggest the following definition for efficient learning.
If the
algorithm for learning $C$ asks $poly(\OPT_A(C))$ queries (rather than
$poly(\OPT_A(C)$ $,n)$) and runs in time $poly(\OPT_A(C),n)$, then we call the
class {\it efficiently learnable}\footnote{We will not use
the term ``polynomially learnable'' for this case to avoid confusion with the
definition in the literature.}.

Another concern with this new definition is that
in many areas, (such as combinatorial group testing and game theory) membership query is
considered to be very costly.
Therefore, one must find polynomial time learning algorithms
that ask a minimum number of queries.
Therefore, we will introduce here two other definitions: If there is
a learning algorithm for $C$ that asks
$\OPT_A(C$ $)^{1+o(1)}$ queries and runs in time $poly(\OPT_A(C),n)$, then
we call the class {\it almost optimally learnable}.
If there is a learning algorithm for $C$ that asks
$O(\OPT_A(C))$ queries and runs in time $poly(\OPT_A(C)$ $,n)$, then
we call the class {\it optimally learnable}.

In many cases, the query complexity $\OPT_A(C)$
is a function of several parameters that are related to
the class $C$. For example, the query complexity $\Theta(d\log n)$ of $d$-MClause
also depends on $d$. We say that
the query complexity of a learning algorithm is {\it optimal} (resp. {\it almost optimal, efficient or polynomial}) {\it in some parameter} if
assuming the other parameters are constant, the query
complexity of the algorithm is optimal (resp. almost optimal, efficient or polynomial).
So a learning algorithm for $d$-MClause that asks $d\cdot poly(\log n)$ queries is efficient, optimal
in $d$ and efficient in $n$.

We say that the class $C$ is {\it query-polynomially} (resp. query-efficiently,
almost query-optimally or query-optimally)
{\it learnable in time} $T$ if the number of queries is as above (for polynomially,
efficiently, almost optimally and optimally, respectively) but the
time complexity is $T$.

We summarize all the above definitions in the following table:

\bigskip
\centerline{
\begin{tabular}{|l|p{3.3cm}|c|}
{\bf Terminology} & {\bf Query Complexity}& {\bf Time Complexity}\\ \hline
{\bf Polynomially Learnable} & &\\
or {\bf Learnable} & $\poly(\OPT_A(C),n)$ & $\poly(\OPT_A(C),n)$\\ \hline
{\bf Efficiently Learnable} & $\poly(\OPT_A(C))$ & $\poly(\OPT_A(C),n)$\\ \hline
{\bf Almost Optimally Learnable} & $\OPT_A(C)^{1+o(1)}$ & $\poly(\OPT_A(C),n)$\\ \hline
{\bf Optimally Learnable} & $O(\OPT_A(C))$ & $\poly(\OPT_A(C),n)$\\ \hline
{\bf Optimally Learnable in} $n$& $O(\OPT_A(C))$ & $\poly(\OPT_A(C),n)$ \\
&{\small {when the other parameters are constant}} & \\ \hline
{\bf Query-Optimally Learnable } & $O(\OPT_A(C))$ & $T(n)$ \\
{\bf in time} $T(n)$ & &\\ \hline
\end{tabular}}

\subsection{Strongly Polynomially, Efficiently and Optimally Learnable}\label{Strong}
Let $C$ be a class of functions. Suppose there is an integer parameter $r$
and classes $C_r$ such that $C=\cup_{r\ge 0} C_r$.
We say that $C$ is
{\it strongly $r$-polynomially learnable} if there is a learning algorithm
for $C$ such that, for every target function $f\in C_r$, the algorithm runs in
time $poly(\OPT(C_r),n)$ and asks at most $poly(\OPT(C_r),n)$ queries.
In the same way as in the above subsection we define {\it strongly $r$-efficiently
learnable}, {\it strongly almost $r$-optimally learnable}, {\it strongly $r$-optimally
learnable} and {\it $r$ learnable in time $T$}.

For example, it is known that $\OPTA(d$-MClause$)=\Theta(d\log n)$.
Obviously, MClause$=\cup_{d\ge 0}d$-MClause. For $f\in$MClause,
let $d(f)$ be the minimum integer such that $f\in d(f)$-MClause.
That is, $d(f)$ is the number of relevant variables of $f$.
The class MClause is
adaptively strongly $d$-optimally learnable if there is an adaptive algorithm
that with a target function $f$ that is MClause, the algorithm runs in time
$poly(n)$ and asks $\Theta(d(f)\log n)$ queries.
That is, the algorithm runs
in time $\Theta(d\log n)$ even when $d$ is not known to the learner.

Recall that, for a function $f:\{0,1\}^n\to \{0,1\}$, we say that $x_i$ is {\it relevant}
in $f$ if there is an assignment $a$ such that $f(a)\not= f(a+e_i)$, where
$\{e_j\}_j$ is the standard bases, and $+$ is
the bitwise exclusive or. That is, if $f$ depends
on the variable $x_i$. We say that $x_i$ is {\it irrelevant} in $f$ if it is not relevant in~$f$.
For a class $C$, let $C_d$ be the set of all functions
in $C$ that have at most $d$ relevant variables.
Then $C=\cup_{d\ge 0} C_d$. A strongly $d$-efficiently learnable class~$C$
will be called {\it strongly attribute-efficiently learnable}.
The same definition applies for attribute-polynomially, almost attribute-optimally, attribute-optimally
and attribute learnable in time $T$.

The definition of ``strongly attribute-efficient'' in~\cite{BHL95}
is equivalent to our definition of ``strongly attribute-optimally learnable in $n$''.

\subsection{Testing Problems}
\label{Testing Problems}

The following problems are also considered in the literature
\begin{enumerate}
\item {\bf Equivalence Testing}: Given two teachers where each one
has a function from~$C$. The learner can ask each one membership queries.
Test whether the two functions
are equivalent. The minimum number of queries is denoted by $\OPT^{\ET}_A$ $(C)$.

For the
non-adaptive algorithms, this is equivalent to constructing a set of assignments $A$
such that for every $f,g\in C$, $f\not= g$ there is
$a\in A$ such that $f(a)\not= g(a)$. Such a set is called an
{\it equivalent test set} or {\it universal identification sequence}~\cite{GKS93}. Obviously,
for deterministic algorithms,
\begin{eqnarray}
\OPTNA^{\ET}(C)=2\cdot \OPT_\NA(C).\label{xyz}
\end{eqnarray} Also, it is easy to show that.
$$\OPTA^{\ET}(C)= 2\cdot \OPT_\A(C).$$
We give a proof sketch of the latter for completeness

\begin{proof} Let ${\cal A}$ be an adaptive algorithm that learns $C$.
We run ${\cal A}$. Each time it asks a membership query $a$, we ask both teachers
that membership query. Since ${\cal A}$ learns $C$, for
some assignment we get different answers. Therefore $\OPTA^{\ET}(C)\le 2\cdot \OPT_\A(C).$

Let ${\cal B}$ be an adaptive algorithm for identity testing the class $C$.
Let $T$ be the teacher with the target function $f\in C$. We run ${\cal B}$ with $T$
and a dummy teacher $T'$ that always gives the same
answer as $T$ as long as there is a function $g\in C\backslash \{f\}$
that is consistent with $f$ on all the answers to the queries.
Finally, no other function is consistent with $f$ and $f$ is uniquely determined.
Therefore $\OPTA^{\ET}(C)\ge 2\cdot \OPT_\A(C).$  \qed
\end{proof}

\item {\bf Identity Testing from $H$} and {\bf Teaching Dimension}: Given a teacher that has
a function $f$ from $C$. Given a function $h\in H\supseteq C$. The learner has $h$
and can ask the teacher membership queries.
Test whether
$f= h$.

A non-adaptive algorithm is equivalent to constructing, for every $h\in H$,
a set of assignments $A_h$
such that for every $f\in C$, $f\not= h$ there is
$a\in A_h$ such that $f(a)\not= h(a)$. Such a set is called an
{\it identity test set} for $h$ with respect to $C$.

Notice that if $h\in C$ then an identity testing set for $h$ uniquely determines~$h$.
If $h\not\in C$ then an identity testing set for $h$ gives a proof
that $h$ does not belong to~$C$. Therefore, we will also call this set
{\it membership test set}.

The maximum, over all the functions $h\in H$, of the minimum size identity test set for $h$
is denoted by $\TD(H,C)$ or $\OPT^{\IT}_\NA(H,C)$. When $H=C$ the set $A_h$ is called a {\it teaching set},
and $\TD(C,C)$ is denoted by $\TD(C)$ or $\OPT^{\IT}_\NA(C)$ and is called the {\it teaching dimension} of the concept $C$,~\cite{GK95,GRS89,SM91}.

We have $\OPT_\A(C)\ge \TD(C)$. To show this fact, suppose ${\cal A}$ is an algorithm
that adaptively learns $C$. Run the algorithm with the target $h$.
The set of all the queries asked is a teaching set for $h$.

For studies in teaching dimension see the
following and reference therein~\cite{GK95,GRS89,H94b,K99,N87,RS94,S01,SDHK91,SM91}.

\item
{\bf Constant Testing}: Given a teacher that has
a function $f$ from $C$.
Test, with membership queries, whether $f$ is not a constant function.
This is equivalent to constructing a set of assignments $A$
such that, for every non-constant function $f\in C$, there is
$a,b\in A$ such that $f(a)\not= f(b)$. Such a set is called a
{\it constant test set for $C$}. $\OPTCS(C)$ denotes the minimum size constant test set
for $C$.

\item
{\bf Zero Testing}: Given a teacher that has
a function $f\not=0$ from $C$. Test whether $f= 0$ with membership
queries. This is
equivalent to constructing a set of assignments $A$
such that, for every non-zero function $f\in C$, there is
$a\in A$ such that $f(a)\not= 0$. Such a set is called a
{\it zero test set} or a {\it hitting set for $C$}. The minimum size hitting set
for~$C$ is denoted by $\OPTHS(C)$. It is easy to show that
$$\TD(H,C)=\max_{h\in H} \OPTHS(C-h)$$
and
\begin{eqnarray}\label{nax}
\OPT_\NAD(C)=\OPT^\HS(C-C)
\end{eqnarray}
where $C-h=\{f-h\ |\ f\in C\}$ and $C-C=\{f-g\ |\ f,g\in C\}$.
For Boolean functions, $f$ and $g$, $f-g=f\oplus g$ where $\oplus$ is the exclusive
or operation.
\end{enumerate}

See other results in the above references and the references therein.

\section{Bounds on OPT for Boolean Functions and Algorithms}\label{Se2}
In this section, we give some bounds for $\OPT_A(C)$ for classes of
Boolean functions~$C$ and some
exponential time algorithms that are query-efficient for any class $C$.

\subsection{OPT for Adaptive Algorithms}\label{OAA}

We first state the following information-theoretic lower bound for
deterministic learning algorithm. Throughout the paper, we write $\log x$ for $\log_2x$.

\begin{lemma}\label{FiL} Let $C$ be a class of Boolean functions. Any
deterministic
learning algorithm for $C$ must ask at least $\log |C|$ membership queries. That is
$$\OPT_\A(C)\ge \log |C|.$$
\end{lemma}

In fact, the bound is also true for Monte Carlo and Las Vegas algorithms.
See for example~\cite{ABM14}
\begin{lemma} \label{lvlb} Let $C$ be any class of Boolean functions. Any
Monte Carlo (and therefore, Las Vegas) randomized
learning algorithm that learns $C$ with probability
at least $3/4$ must ask at least $\log |C| -1$ membership queries. That is
$$\OPT_{\A,\LV}(C)\ge \OPT_{\A,\MC}(C)\ge \log |C|-1.$$
\end{lemma}

We now give upper and lower bounds for $\OPT_A(C)$ using the following combinatorial measure
that is defined in~\cite{H95,M83}.
Let $C$ be a class of Boolean functions $f:X\to \{0,1\}$.
Let $h:X\to \{0,1\}$ be {\it any} function.
We say that
a set $T_h\subseteq X$ is a {\it specifying set for $h$} with respect to $C$
if $$|\{f\in C\ |\ (\forall x\in T_h) h(x)=f(x)\}|\le 1.$$ That is, there is
at most one concept in $C$ that is consistent with $h$ on $T_h$. Denote $T(C,h)$
the minimum size of a specifying set for $h$ with respect to $C$.
The {\it extended teaching dimension} of $C$ is
$$\ETD(C)=\max_{h\in 2^X} T(C,h).$$
Let $h:X\to \{0,1\}$ be any function. For every $f\in C$, $f\not=h$ there is an assignment $x_f\in X$ such that $f(x_f)\not=h(x_f)$. Therefore $\{x_f\ |\ f\in C\backslash \{h\}\}$ is a specifying set for $h$ and $T(C,h)\le |C|$. Therefore
\begin{eqnarray}\label{ElC}
\ETD(C)\le |C|.
\end{eqnarray}
Notice that for every Boolean function $h\not\in C$
if $T_h$ is a specifying set for $h$ and $f\in C$ is
a function that is consistent with $h$ on $T_h$ then adding
an assignment $a$ to $T_h$ where $f(a)\not=h(a)$ gives an identity testing set for $h$ with respect to $C$.
Therefore,
$$\ETD(C)\le \TD(2^X,C)\le \ETD(C)+1.$$

In~\cite{H95,M83}, Moshkov proves the following bounds. Here, we will give
another proof that gives, asymptotically, the same upper bound for $\OPT_\A(C)$.
\begin{lemma}~\cite{H95,M83} \label{LBo} Let $C$ be any class of Boolean functions. Then
\begin{eqnarray}\label{seqr}
\OPT_\A(C)\le  \frac{\ETD(C)}{\log \ETD(C)}\log|C|+\ETD(C) \le \frac{2\cdot \ETD(C)}{\log \ETD(C)}\log|C|\end{eqnarray} and
$$\OPT_\A(C)\ge \max(\ETD(C),\log |C|).$$
\end{lemma}
\begin{proof} The second inequality in (\ref{seqr}) follows from (\ref{ElC}).

Consider the following algorithm.
After the $i$th query, the algorithm defines a set $C_i\subseteq C$
of all the functions that are consistent with the membership queries that were asked so far.
Consider any $0<\epsilon<1/2$.
Now the algorithm searches for an assignment $a\in X$
such that $$\epsilon |C_i|\le |\{f\in C_i\ |\ f(a)=0\}|\le (1-\epsilon) |C_i|.$$
If such $a\in X$ exists, then it asks the membership queries $\MQ(a)$. Define $C_{i+1}=\{f\in C_i\ |\ f(a)=\MQ(a)\}$. Obviously,
in that case,
$$|C_{i+1}|\le (1-\epsilon) |C_i|.$$

If no such $a\in X$ exists, then the algorithm finds a specifying set $T_h$ for
$h:={\rm Majority}(C_i)$, where ``Majority'' is the majority function. It then
asks membership queries for all the assignments in $T_h$. If the answers are consistent
with $h$ on $T_h$, then there is a unique concept $c\in C_i$ consistent
with the answers, and the algorithm outputs this concept.
Otherwise, there is $a\in T_h$ such that $\MQ(a)\not= h(a)$. It is easy to see that
in that case
$$|C_{i+1}|\le \epsilon |C_i|.$$

Denote by $k$ the number of times when the algorithm is left without
such an $a \in X$. Then the number rounds when it does find such an $a
\in X$ is $$\frac{\log |C| - k \log(1/\epsilon)}{\log(1/(1-\epsilon))}$$
and so the number of queries is upper bounded by
\begin{eqnarray}
\frac{\log |C| - k \log(1/\epsilon)}{\log(1/(1-\epsilon))} &+& \ETD(C)k\nonumber\\
 & =& \frac{\log |C|)}{\log(1/(1-\epsilon))} + k \left(\ETD(C) -
  \frac{\log(1/\epsilon)}{\log(1/(1-\epsilon))}\right)\label{neww}
\end{eqnarray}
When $\epsilon$ is chosen so that the last term becomes $0$, then $\epsilon$ is roughly $\log \ETD(C)/$ $\ETD(C)$. In case of this $\epsilon$, (\ref{neww}) becomes
\begin{eqnarray*}
(1+o(1))\frac{\ETD(C)}{\log \ETD(C)} \log|C|.\qed
\end{eqnarray*}
\end{proof}

In \cite{H95,M83}, Moshkov gives, for any two integers
$t$ and $\ell$, an example of a class $C_{t,\ell}$ where $$\OPT_\A(C_{t,\ell})=
\Omega\left(\frac{\ETD(C_{t,\ell})}{\log \ETD(C_{t,\ell})}\log |C_{t,\ell}|\right).$$ So the upper bound in the above lemma is
the best possible.

See also the other dimensions, bounds, and techniques in~\cite{A88,A87,A01,AHHP98,H94,MT92}.

\subsection{Constructing Adaptive Algorithms}\label{CAA}

Given a class of Boolean functions $C$ as an input.
Can one construct an algorithm that learns $C$ with $\OPTA(C)$ queries?
Obviously, with unlimited computational power, this can be done
so the question is:
How close to $\OPTA(C)$ can one get
when polynomial time (or any other time) is allowed for the construction?

An exponential time algorithm follows
from the following
\begin{eqnarray}\label{algforopt}
\OPTA(C)=\min_{x\in X} \max(\OPTA(C_{x,0}),\OPTA(C_{x,1}))
\end{eqnarray}
where $C_{x,\xi}=\{f\in C\ |\ f(x)=\xi\}$.
See also~\cite{AGMMPS93,G72}. Can one do it in $poly(|C|,|X|)$ time?
Hyafil and Rivest, \cite{HR76}, show that the problem of finding $\OPTA$ is NP-Complete.
Laber and Nogueira, \cite{LN04}, show that this problem does not admit an $o(\log |C|)$-approximation unless P=NP.
The reduction of Laber and Nogueira, \cite{LN04}, of set cover to this problem with
the inapproximability result of Dinur and Steurer~\cite{DS14} for set cover
implies that it cannot be approximated to $(1-o(1))\cdot \ln |C|$ in polynomial time
unless P=NP.

The query complexity of the algorithm in Lemma~\ref{LBo} is within a factor of
$$\min\left(1+\frac{\log |C|}{\log \ETD(C)}\ \ \ , \ \ \ \frac{\ETD(C)}{\log {|C|}}+\frac{\ETD(C)}{\log {\ETD(C)}}\right)$$ from OPT. However, unfortunately, the problem of finding a minimum size specifying set for $h$ is
NP-Hard, \cite{ABCS92,GK95,SM91}.

In \cite{AMMRS93}, Arkin et al. gives the following algorithm:
Let $C_i$ be the set of functions that are consistent with
the answers to the first $i$ membership queries.
The $i+1$th membership query of the algorithm is
an assignment $a$ that partitions the set $C_i$ as evenly as possible, that is, an assignment $a$ that maximizes
$\min(|\{f\in C_i|f(a)=0\}|,|\{f\in C_i|f(a)=1\}|)$.
Arkin et al. show in \cite{AMMRS93} that
the query complexity of this algorithm is within a factor of $c\ln|C|$ from $\OPT$
for some $c>1$.
Moshkov in~\cite{M04} gives the exact ratio of $\ln |C|$.
This algorithm runs in time $poly(|C|,|X|)$. In particular, we have the following result.
Here, we will give a very simple proof
\begin{lemma} \cite{AMMRS93}. There is a learning algorithm that runs in time $poly(|C|,|X|)$
and learns $C$ with at most $$(\ln |C|)\OPT_\A(C)\le 0.693\cdot \OPTA(C)^2$$ queries.

In particular, any class $C$ is
adaptively query-efficiently learnable in time poly$(|C|,|X|)$.
\end{lemma}
\begin{proof} First, we have
$$\OPTA(C)\ge \frac{|C|}{\max_a \min(|C_{a,0}|,|C_{a,1}|)}$$
where $C_{x,\xi}=\{f\in C\ |\ f(x)=\xi\}$.
This follows from the fact that every membership query can eliminate at most
$\max_a \min(|C_{a,0}|,|C_{a,1}|)$ functions from $C$. Therefore,
there is $b\in X$ such that $$\max_a \min(|C_{a,0}|,|C_{a,1}|)=\min(|C_{b,0}|,|C_{b,1}|)\ge \frac{|C|}{\OPTA(C)}.$$
Thus, using $b$ as the first query, eliminates at least $|C|/\OPTA(C)$ functions from $C$.
After the $\ell$th query the number of functions that remain (consistent with
the answers to the membership queries) is at most
$$|C|\left(1-\frac{1}{\OPTA(C)}\right)^\ell.$$ So
the number of queries of the algorithm is at most $\ell=(\ln |C|) \OPTA(C)$.
By Lemma~\ref{FiL}, $\ell\le 0.693\cdot\OPT_{\AD}(C)^2$.\qed
\end{proof}

Bshouty et al., \cite{BCGKT96}, show that using the NP-oracle
all Boolean classes are efficiently learnable in randomized expected polynomial time. See other
results in \cite{BCGKT96}.

\subsection{OPT for Non-Adaptive Algorithms}
In this subsection, we give some bounds for non-adaptive learning classes Boolean functions.
The following result follows from (\ref{xyz}) and (\ref{nax})
\begin{lemma} \label{xor} We have
\begin{enumerate}
\item $$\OPTNA(C)\ge \OPTA(C).$$

\item $\OPTNA(C)$
is equal to half the minimum size equivalent test set for $C$.

\item The set $A$ is an equivalent test set for $C$ if and only if
$A$ is a hitting set for $C\oplus C:=\{h\oplus g\ |\ h,g\in C\}$.
That is,
$$\OPTNA(C)=\OPTHS(C\oplus C).$$
\end{enumerate}
\end{lemma}

We now prove
\begin{lemma} \cite{GK95,K72}. There is a learning algorithm that runs in time $poly(|X|,$ $|C|)$
and finds a hitting set for $C$ of size at most
\begin{eqnarray}\label{seq}
(\ln|C|)\OPTHS(C).
\end{eqnarray}

In particular,
\begin{enumerate}
\item There is a non-adaptive learning algorithm that runs in time $poly(|X|,|C|)$
and learns $C$ using at most
\begin{eqnarray}\label{Seq}
(2\ln|C|)\OPTNA(C)\le 1.386\cdot \OPTNA(C)^2
\end{eqnarray} queries.

\item Any class $C$ is
non-adaptively query-efficiently learnable in time poly$(|C|,|X|)$.
\end{enumerate}
\end{lemma}
\begin{proof} Define for every $x\in X$ the set $C_x$ of all the functions
$f\in C$ such that $f(x)=1$. Now the
hitting set problem is equivalent to the set cover problem, i.e.,
find the minimal number of elements $S\subset X$ such that $\cup_{x\in S}C_x=C$.
It is known that the greedy algorithm that at each stage, chooses the set that
contains the largest number of uncovered elements, achieves an approximation ratio of $\ln |C|$,~\cite{C79}.

Now (\ref{Seq}) follows from (\ref{seq}), Lemma~\ref{FiL} and {\it 3} in Lemma~\ref{xor}.\qed
\end{proof}
The above reduction shows that
the problem of finding a small hitting set is equivalent to finding a small
set cover, and therefore, the minimum hitting set problem
cannot be approximated in polynomial time to within of
factor of less than $(1-o(1))\ln |C|$,~\cite{AMS06,DS14,F98,RS97}.

A Hitting set for $C\oplus C$ is also a hitting set for $C$ except probably
one function. This follows from the fact that if $A$ is
a hitting set for $C\oplus C$ then for each two distinct functions $f,g\in C$
there is $a\in A$ such that $f(a)\not=g(a)$ and
therefore $a$ hits one of them. This implies that there is
no learning algorithm that runs in time $poly(|X|,$ $|C|)$
and non-adaptively learns $C$ with less than
$(\ln|C|)\OPT_\NA(C)$ unless P$=$NP.

\subsection{OPT for Classes of Small VC-dimension}
We have seen in the previous subsection that the query complexity of non-adaptive learning
the class $C$ is equal to the minimum size hitting
set for $C\oplus C$ and finding a small hitting set is equivalent
to finding a small set cover.
We now give another way to construct a small hitting set
for classes with small Vapnik-Chervonenkis ($\VCdim$) dimension.
We first define the $\VCdim$ of a class~$C$.

For a class $C$ and a set $S\subseteq X$, we say that $S$ is {\it shattered by} $C$ if for any $T\subseteq S$ there is a function $f\in C$ such that for all $a\in T$
we have $f(a)=0$ and for all $a\in S\backslash T$ we have $f(a)=1$.
The $\VCdim$ of a class $C$, $\VCdim(C)$, is the maximum integer $d$ such that there is a set $S$ of size $d$ that is shattered by~$C$.

Another way to construct a hitting set for $C$ is by choosing a distribution $P$ on the domain $X$ and then repeatedly chooses elements $x\in X$ according to the distribution $P$ until we get a hitting set. This approach can also be used to prove an upper bound for the hitting set size.
\begin{lemma}
Suppose one can define a distribution $P$ over $X$ such that for every $f\in C$, $\Pr_P[f(x)\not=0]\ge \epsilon_P$. Then a set of
$$ T_P(C):=O\left(\min\left(\frac{1}{\epsilon_P}\log |C|\ , \ \frac{\VCdim(C)}{\epsilon_P} \log \frac{1}{\epsilon_P} \right)\right)$$ elements chosen according to the distribution $P$ is a hitting set
 for $C$ with probability at least $1/2$.

In particular,
$$\OPTHS(C)\le \TP(C):=\min_P T_P(C).$$
\end{lemma}
\begin{proof} The result follows from the $\epsilon$-net theorem~\cite{HW87}. See also Chapter~13
in~\cite{AS08}.
\qed
\end{proof}

In particular, this also gives the following upper bound for $\OPTNA$
\begin{lemma}
Suppose one can define a distribution $P$ over $X$ such that for every $f,g\in C$ we have $\Pr_P[f(x)\not=g(x)]\ge \epsilon$. Then a set of
$T_P(C\oplus C)$ elements chosen according to the distribution $P$ is
 an equivalent test set for $C$ with probability at least $1/2$.

In particular,
$$\OPTNA(C)\le \TP(C\oplus C).$$
\end{lemma}
Br\"{o}nnimann and Goodrich, \cite{BG95}, show that there is an algorithm that runs in time $poly(|X|,|C|)$
and finds a hitting set for $C$ of size at most $$O(\VCdim(C)\log(\OPTHS(C))\cdot \OPTHS(C).$$
See also~\cite{ERS05}.

\section{Reductions}\label{Reduction}
In this section, we give some reductions that change
an existing algorithm to an algorithm with a better query complexity
and an algorithm in another learning model to an algorithm that learns from membership queries.
We will show in the sequel how to apply those reductions.

\subsection{Reductions for Adaptive Algorithms}\label{AtoA}
In this subsection, we show some reductions from one exact adaptive learning algorithm to another
one. Those reductions change the query complexity to be optimal in some of the parameters of the class.

For a class $C$ of functions $f:\{0,1\}^n\to \{0,1\}$
we say that $C$ is {\it projection closed} if for any $f\in C$,
$1\le i\le n$ and $\xi\in\{0,1\}$ we have $f|_{x_i\gets \xi}\in C$.
That is, projecting any variable $x_i$ to any value $\xi\in\{0,1\}$ keeps the function in the class $C$.
We say that the class $C$ is {\it embedding closed} if for any $f\in C$
and any $1\le k,i\le n$ we have $f|_{x_i\gets x_k}\in C$. We note here
that almost all the classes considered in the literature are projection and embedding closed.

Here we use the parameter $r$ for the number of relevant
variables. For class $C$, the class $C_r$ contains all the functions
in $C$ with at most $r$ relevant variables. All the classes below are
projection closed. It is easy to show that for projection closed class $C$, $C_r\supseteq C_1$ and $C_1$ contains
at least $n$ literals. Therefore, $\OPT_\AD(C_r)\ge \log n$.

In~\cite{BHL95}, Blum et al. show
\begin{lemma}\label{BHL0}\cite{BHL95} Let $C$ be a class that is projection closed.
If $C$ is adaptively learnable in time $T(n)$ with
$Q(n)$ queries
and there is a constant testing set for $C_r$ of size $P(n,r)$
that can be constructed in time $t(n)$ then
$C_r$ is adaptively learnable in time $poly(n,T(r),Q(r),t(n),P(n,r))$
with  $$O(r\cdot Q(r)P(n,r)+r\cdot \log n)$$ queries.

In particular, if $P(n,r)=O(\log n)$ (when the other parameters
of the class $C$ are constants), then $C_r$ is optimally learnable in $n$.
\end{lemma}

The results in~\cite{BH98} show that if the class is also
embedding closed then $P(n,r)=O(r^6\cdot Q(r^2)\log n)$.
Therefore, with Lemma~\ref{BHL0}, we get the query complexity
$$O((r^7\cdot Q(r)\cdot Q(r^2)+r) \log n).$$

When the class is also embedding closed, Bshouty and Hellerstein,\cite{BH98}, show

\begin{lemma}\cite{B14b,BH98}\label{kk1} Let $C$ be a projection and embedding closed class.
If $C$ is adaptively learnable in time $T(n)$ with $Q(n)$ queries,
then $C_r$ is adaptively learnable in time $poly(T(n))$
with $$O(r^6\cdot Q(r^2)\log n)$$ queries.

In particular, if $C$ is learnable then $C_r$ is optimally learnable in $n$.
\end{lemma}

For the randomized algorithm, we have
\begin{lemma}\cite{BH98} Let $C$ be a projection and embedding closed class.
If $C$ is adaptively learnable with a Monte Carlo algorithm in time $T(n)$ with $Q(n)$ queries
then $C_r$ is adaptively learnable with a Mote Carlo algorithm in time $poly(T(n))$
with $$O(r^3\cdot Q(2r^2)\log n)$$ queries.
\end{lemma}

\subsection{Reductions for Strong Adaptive Algorithms}\label{AtoSA}
In some cases, one can achieve strong learning.
The following result is from~\cite{BHL95}. See the definition
of strong learning in Subsection~\ref{Strong}.
\begin{lemma}\label{BHL}\cite{BHL95} Let $C$ be a class that is projection closed.
If $C$ is adaptively learnable in time $T(n)$ with
$Q(n)$ queries
and there is a constant testing set for $C$ of size $P(n)$
that can be constructed in time $t(n)$ then
$C$ is adaptively strongly-attribute learnable in time $poly(n,T(r),Q(r),t(n),P(n))$
with  $$O(r\cdot Q(r)P(n)+r\cdot \log n)$$ queries where $r$ is the number
of relevant variables of the target.

In particular, if $P(n)=O(\log n)$ (when the other parameters
of the class $C$ are constants), then $C$ is strongly attribute-optimally learnable in $n$.
\end{lemma}

Since $\{0^n,1^n\}$ (the all zero and all one assignments)
is a constant testing set for all the monotone functions,
$P(n)=2$, in particular, we have
\begin{lemma}~\label{RedM}\cite{BHL95} Let $C$ be a projection closed class that contains monotone functions.
If $C$ is adaptively learnable in time $T(n)$ with
$Q(n)$ queries, then
$C$ is adaptively strongly attribute learnable in time $poly(n,T(r),Q(r))$
with  $$O(r\cdot Q(r)+r\cdot \log n)$$ queries.

In particular, $C$ is strongly attribute-optimally learnable in $n$.
\end{lemma}

So with the above results, one can change
algorithms that learn $C$ to algorithms that
optimally and strongly attribute-optimally learn $C$ in $n$.

\subsection{Reductions for Non-Adaptive Algorithms}\label{NAtoNA}
In this subsection, we give a reduction from one exact non-adaptive learning algorithm to another
one that changes the query complexity to be optimal in some of the parameters.

In \cite{ABM15}, Abasi et al. gave the following reduction.
We sketch the proof for completeness
\begin{lemma}\label{DNtoDN} Let $C$ be an embedding closed class such that
$|C_r|=\Omega(n)$. Let $H$ be a class of Boolean functions
and suppose there is an algorithm
that, for input $h\in H$, finds the relevant variables of $h$ in time $R(n)$.

If $C$ is non-adaptively learnable from $H$ (the output hypothesis is from $H$)
in time $T(n)$ with $Q(n)$ membership queries then $C_r$ is non-adaptively
learnable from $H$ in time $poly(n,T(q),R(q))$
with $$O\left(\frac{r^2Q(q)}{\log (q/r^2)}\log n\right)$$ membership queries
where $q\ge 2r^2$ is any integer.

In particular, $C_r$ is non-adaptively optimally learnable in $n$.
\end{lemma}
\begin{proof} We use a {\it perfect hash family} $P$ that map the variables $X=\{x_1,\ldots,x_n\}$ to a
new set of variables $Y=\{y_1,\ldots,y_q\}$ where $q>r^2$.
This family contains $O(r^2/\log (q/r^2)\log n)$ hash functions and ensures that for
almost all the hash functions, different relevant variables of the target
are mapped to different variables in $Y$. It also insures that
for every non-relevant variable $x_i$, almost all the hash functions map the relevant variables and $x_i$ to different variables in $Y$.

We learn $f(\phi(X))$ for each hash function $\phi \in P$ with $Q(q)$
membership queries (when possible) and use the
majority rule to find the relevant variables and to recover the target function.\qed
\end{proof}

See also \cite{BC16} for a reduction for the randomized non-adaptive learning.

\subsection{Reductions from the Exact Learning Model}\label{REL}
In this subsection and the next one, we give two other models and show some
conditions in which learning in those
models can be reduced to learning from membership queries only.

The first learning model is the {\it exact learning model
from membership and equivalence queries}.
In this model, the goal is to learn the target function exactly
with membership queries
and equivalence queries.
In the Equivalence Query (EQ) model, \cite{A88}, the learning algorithm sends
the teacher a hypothesis $h$ from some class of hypothesis
$H\supseteq C$. The teacher answers ``YES'' if $h$ is equivalent
to the target $f$, otherwise, it provides the learner a
counterexample, i.e., an assignment $a$ where $f(a)\not= h(a)$. We say that
a class $C$
is exactly learnable from $H$ from membership and equivalence queries
in time $t$, $m$ membership queries and $e$ equivalence queries if
there is an algorithm that runs in time $t$, asks at most $m$ membership queries
and at most $e$ equivalence queries and outputs $h\in H$ that is
equivalent to $f$.

There are several polynomial time exact learning algorithms available in the literature
that learns from membership and equivalence queries for classes mentioned in
this survey and others. Classes such as Monotone DNF \cite{A87},
DFA \cite{A87b}, Conjunction of Horn Clauses \cite{AFP92}, $O(\log
n)$-term DNF \cite{B95,B97,BR95,K97}, read-twice DNF \cite{AP91},
CDNF \cite{B95}, decision trees \cite{B95,BBBKV00}, Boolean
multivariate polynomial \cite{SS96,BBBKV00}, multiplicity automata
\cite{BBBKV00}, read-once formula \cite{BHH95a} and geometric objects~\cite{B98}. See also
references therein and \cite{SST10}. Some of the algorithms are proper (i.e., $H=C$)
and others are non-proper.

The reduction from this model to the model of learning
from membership queries only is done as follows.
\begin{lemma}\label{RedEQ} Let $C$ and $H\supseteq C$ be two classes of functions.
Suppose $C$ is learnable from $H$ from membership and equivalence queries
in time $t$, $m$ membership queries and $e$ equivalence queries. Then
\begin{enumerate}
\item If for every $h\in H$
an identity testing set $S_h$ for $h$ with respect to $C$ can be constructed
in time $t'$ and $|S_h|\le s$ then $C$ is learnable in time $poly(t',t,m,e,s)$
with $m+es$ membership queries.

\item If an equivalent test set $S$ for $H$ can be constructed in time $t'$
and $|S|=s$ then $C$ is learnable in time $poly(t',t,m,e,s)$ with $m+s$ membership queries.
\end{enumerate}
\end{lemma}
\begin{proof} In {\it 1}, the algorithm replaces each equivalence query for $h$ to
membership queries to all the assignments in $S_h$. If $h$ is consistent with $f$ on
$S_h$ then $h= f$ and the algorithm outputs $h$. Otherwise, there is $a\in S_h$ such
that $f(a)\not=h(a)$ and $a$ can be used as an answer to the equivalence query.

In {\it 2}, the algorithm first constructs $S$ and asks membership queries
to all the assignments in~$S$.
Then for each equivalence query with $h$, it finds $a\in S$ such that $h(a)\not= f(a)$
and returns the answer $a$ to the algorithm. If no counterexample exists then,
it outputs $h$.\qed
\end{proof}

A hardness result in learning a class $C$ from equivalent and membership queries
does not imply hardness in learning $C$ from membership queries only.
For example, the hardness result of learning read-thrice DNF, $C$, in \cite{AHP92}
cannot be used as a hardness result for this class in the exact
learning model from membership queries. This is because, in our definition of
efficiency, the query complexity is allowed to be polynomial in $\OPT_A(C)$
which might be exponential in $n$.
For example, it is easy to see that $\OPTA($read-thrice DNF$)=2^n$ (even for one term)
so this class is optimally learnable (by asking all the queries $\{0,1\}^n$).
So hardness in learning
in this model does not imply hardness in learning from membership queries.
On the other hand,
hardness results for proper learning,~\cite{ABFKP08}, do give
hardness results for proper learning from membership queries.

\subsection{Reductions from the PAC Learning Model}
In this section, we provide a reduction from PAC-learning with membership queries
to learning from membership queries only.

In the {\it probably approximately correct learning model}
(PAC learning model), \cite{V84}, with membership queries the teacher has a function $f:X\to \{0,1\}$
from some class $C$. The learner can ask the teacher membership queries and is required to learn
a function that is $\epsilon$-close (defined below) to the target with high probability.
Let $P$ be a distribution on the domain $X$.
We say that an algorithm ${\cal A}$
{\it PAC-learns $C$ from $H$ with membership queries according
to the distribution} $P$ if the algorithm ${\cal A}$, for the input $\epsilon<1$ and $\delta<1$,
with probability at least $1-\delta$,
outputs a function $h\in H$ that is $\epsilon$-close to $f$, i.e.,
$\Pr_P[f(x)\not= h(x)]\le \epsilon$.

There are many PAC learning algorithms with membership queries
in the literature for classes mentioned in
this survey and others. Classes such as decision trees
and multivariate polynomials under distributions
that support small terms, \cite{BM02},
DNF under the uniform distribution, \cite{F07,J97}, constant depth circuits under the uniform distribution, \cite{LMN93} and
intersections of halfspaces, \cite{KP98}.
If a class is learnable from equivalence and membership queries, then it is
PAC-learnable with membership queries according to any distribution, \cite{A87}.

We say that $H$ is a class of {\it distance} $\eta$ {\it with respect to the distribution} $P$ if
for every $h_1,h_2\in H$, $h_1\not= h_2$ we have $\Pr_P[h_1(x)\not= h_2(x)]> \eta$.
The following lemma shows how to change a PAC-learning algorithm to a learning algorithm
with membership queries.

\begin{lemma}\label{RedPAC} Let $C$ and $H\supseteq C$ be two classes of functions.
Suppose $C$ is PAC-learnable with membership queries from $H$ according to the distribution $P$
in time $t(\epsilon,\delta)$ and $m(\epsilon,\delta)$
membership queries. If $H$ is a class of distance
$\eta$ with respect to the distribution $P$, then
$C$ is randomized Monte Carlo learnable with $m(\eta,\delta)$ membership queries in time $t(\eta,\delta)$.
\end{lemma}
\begin{proof} Run the algorithm with $\epsilon=\eta$. Let $h$ be the output.
Then with probability at least $1-\delta$ we have $\Pr_P[f\not=h]\le \eta$.
Since $C\subseteq H$ we also have $f\in H$ and since $H$ is a class of distance
$\eta$ we must have $f= h$.\qed
\end{proof}

\section{Learning $d$-MClause and Group Testing}\label{GTest}
In this section, we give an example of a class that has been extensively studied in the literature.
Consider the class $d$-MClause. That is the class of monotone
clauses with at most $d$ variables.

\subsection{Group Testing and Applications}
In group testing (or pooling design), the task is to determine
the positive members $F=\{o_{i_1},\ldots,o_{i_k}\}$, $k\le d$, of a set of objects ${\cal O}=\{o_1,\ldots,o_n\}$
by asking queries of the form ``does
the subset $S\subseteq {\cal O}$ contain a positive object?''
That is, ``does $F\cap S\not= \emptyset$?''.
A negative answer to
this question informs the learner that all the items belonging to $S$ are non-positive. The aim of group testing is to identify the unknown subset $F$ using
as few queries as possible.

Group testing was originally introduced as a potential approach to economical
mass blood testing \cite{D43}. However it has been proven to be applicable in a variety of problems, including quality control in product
testing \cite{SG59}, searching files in storage systems \cite{KS64}, sequential screening of experimental variables \cite{L62}, efficient contention resolution algorithms for multiple-access communication \cite{KS64,W85}, data compression \cite{HL02}, and computation in the data stream model \cite{CM03}. See a brief history and other applications in~\cite{Ci13,DH00,DH06,ND00}
and references therein.

Group testing is equivalent to learning the class $d$-MClause from membership queries.
We have $n=|{\cal O}|$; the target is $\vee_{o_i\in F}x_i$ and each query ``does $F\cap S\not= \emptyset$?'' is equivalent to a membership query with the assignment $a\in \{0,1\}^n$
where $a_i=1$ if and only if  $o_i\in S$.

\subsection{Known Results for Learning $d$-MClause}
The following table summarizes the results known
for the asymptotic number of membership queries for learning
the class $d$-MClause. We will assume that $d\le \sqrt{n}$.

\bigskip
\begin{tabular}{l||c|c|c|}
& Non-Adaptive & Two-Round & Adaptive \\
 \hline\hline
Deterministic  Upper Bound & ${{d^2}}\log n$ & $d\log n$ & $d\log n$\\
\hline
Deterministic  Poly Time & ${d^2}\log n$ & $d\log^{O(1)} n$ & $d\log n$\\
&&$d^{1+o(1)}\log n$&\\
\hline
Deterministic Lower Bound & $\frac{d^2}{\log d}\log n$ & $d\log n$ & $d\log n$\\

\hline
LV Randomized Poly Time & ${d^2}\log n$ & $d\log^{O(1)} n$ & $d\log n$\\
\hline
LV Randomized Lower Bound & ${d^2} \log n$ & $d\log n$ & $d\log n$\\
\hline
MC Randomized Poly Time & ${d}\log n$ & $d\log n$ & $d\log n$\\
\hline
MC Randomized Lower Bound & ${d}\log n$ & $d\log n$ & $d\log n$\\
\hline
\end{tabular}
\bigskip

There is a (folklore) deterministic adaptive algorithm that runs in polynomial time
and asks $O(d\log n)$ queries. See Subsection~\ref{GTA}.
This implies all the results for the upper bounds in the fourth column of the above table.
The number of functions in $d$-MClause is $${n\choose d}+{n\choose d-1}+\cdots+{n\choose 0}$$
and therefore by Lemma~\ref{lvlb}, any Monte Carlo learning algorithm for $d$-MClause
must ask at least $\Omega(d\log n)$
queries. This implies all the lower bounds in the table except for the non-adaptive lower bound for
the deterministic and Las Vegas algorithm that follows from~\cite{DR82}. See Subsection~\ref{GTNA}.
The upper bound of $$O\left(d\left(\log n+\log \frac{1}{\delta}\right)\right)$$ for the Monte Carlo non-adaptive algorithm
follows from a simple randomized argument. See Subsection~\ref{GTNA}.
Porat and Rothschild, \cite{PR11}, gave the first polynomial time deterministic
non-adaptive learning algorithm that
asks $O(d^2\log n)$ queries.
The deterministic upper bound for non-adaptive learning follows from a simple probabilistic
argument. See Also~\cite{DVPS14,DVPS16}.
The last two results are the deterministic two-round
algorithms with $O(d\log n)$ queries and the polynomial time deterministic two-round learning algorithm
that asks $\min(d\log^{O(1)}n,d^{1+o(1)}$ $\log n)$ queries.
This follows from~\cite{BGV05,C13,I02}. See Subsection~\ref{TRL}.

In particular,
\begin{theorem} We have
\begin{enumerate}
\item The class $d$-MClause is adaptively optimally learnable.

\item The class $d$-MClause is non-adaptively almost optimally learnable and optimally learnable in $n$.

\item The class $d$-MClause is MC randomized adaptively optimally learnable.

\item The class $d$-MClause is two-round efficiently learnable, two-round optimally learnable in $d$
and two-round optimally learnable in $n$.
\end{enumerate}
\end{theorem}

\subsection{Bounds for OPT($d$-MClause)}
Consider a non-adaptive algorithm ${\cal A}$ for learning $d$-MClause and let $x^{(1)},\ldots,x^{(t)}$ be the queries asked by the algorithm. Consider the $t\times n$ matrix $M$ such that its $i$th row is $x^{(i)}$. Let $M_i$ be the $i$th column of $M$.
If the target function is $x_{i_1}\vee \cdots \vee x_{i_{d'}}$, $d'\le d$ then the
vector of all the answers to the queries is $M_{i_1}\vee\cdots \vee M_{i_{d'}}$ (bitwise or).
Let ${\cal M}$ be the set of all $M_{j_1}\vee\cdots \vee M_{j_{d}}$, $j_i\in [n]$ for $i=1,\ldots,d$.
That is, the set of all possible answers to the queries for all possible target functions.
Since
each vector in ${\cal M}$ uniquely determines the target function, the matrix $M$ must satisfy the following property: For every $S_1,S_2\subseteq [n]$, where $|S_1|,|S_2|\le d$ and $S_1\not= S_2$, we have $$\bigvee_{j\in S_1}M_j\not= \bigvee_{j\in S_2}M_j.$$ A matrix that satisfies this property is called $d$-{\it separable} matrix.
Therefore, $\OPTNA(d$-Mclause$)$ is equal to the minimum $t$ such that a $d$-{\it separable} $t\times n$ matrix exists.

A $t\times n$ matrix $M$ is called $d$-{\it disjunct} if for every $d+1$ distinct columns $i_1,\cdots ,i_{d+1}$ there is a row $j$ such that $M_{j,i_{d+1}}=1$ and $M_{j,i_k}=0$ for all $k=1,\ldots, d$. It is easy to show that,~\cite{KS64},
$$\mbox{$M$ is $d$-disjunct $\Rightarrow$ $M$ is $d$-seperable $\Rightarrow$ $M$ is $(d-1)$-disjunct}.$$
Therefore, it is enough to construct a $d$-disjunct matrix.
Using a probabilistic method, \cite{AS08}, it is easy to show that a $d$-disjunct $t\times n$ matrix exists where $t=O(d^2\log n)$.
Just take each entry in the matrix $M$ to be $1$ with probability $1/d$ and $0$ with probability
$1-1/d$ and show that for $t=O(d^2\log n)$ the probability that the matrix is not $d$-disjunct is
less than $1$.
This implies $$\OPTNA(d\mbox{-MClause})= O\left({d^2}\log n\right).$$
See also \cite{DVPS14,DVPS16,NZ88}.
There is also an almost tight lower bound for $t$~\cite{DR82,F96,R94}
$$\OPTNA(d\mbox{-MClause})=\Omega\left( \frac{d^2}{\log d} \log n\right).$$

\subsection{Non-Adaptive Learning $d$-MClause}\label{GTNA}
We now show that the class $d$-MClause is
non-adaptively learnable in polynomial time with $O(d^2\log n)$ queries.
In \cite{PR11}, Porat and Rothschild gave the first polynomial time algorithm for constructing
a $d$-disjunct matrix of size $t\times n$ where $t=O(d^2\log n)$.
Now the learning algorithm is as follows.
We first use the Porat and Rothschild~\cite{PR11} algorithm
to construct a $d$-disjunct $t\times  n$ matrix $M$ of size $t=O(d^2\log n)$ in polynomial time. Set $Z=\{x_1,x_2,\ldots,x_n\}$. Then for every query $M^{(i)}$ (row $i$ in $M$) if $f(M^{(i)})=0$ then
for every $j$ where $M^{(i)}_j=1$, we remove $x_j$ from $Z$.
The remaining variables in $Z$ are the variables that appear in the target. This is because if $f=x_{i_1}\vee \cdots \vee x_{i_{d'}}$ and $i\not\in \{i_1,\ldots,i_{d'}\}$ then the variable $x_i$ will be removed from $Z$ by the row $M^{(j)}$ that assign $0$ to all $x_{i_j}$, $j=1,\ldots,d'$ and one to $x_i$. Such row exists since $M$ is $d$-disjunct. On the other hand, for every row $M^{(k)}$ where $M^{(k)}_{i_j}=1$ for some $j=1,\ldots,d'$ we have $f(M^{(k)})=1$ and, therefore, no variable
in the target is removed from~$Z$. This gives a polynomial time algorithm that asks $O(d^2\log n)$ queries.
Since the lower bound for the query complexity
of non-adaptively learning the class $d$-MClause is $\Omega(d^2\log n/\log d)$,
the class $d$-MClause is
non-adaptively almost optimally learnable and non-adaptively optimally learnable in~$n$.

Closing the gap between the lower bound and upper bound is one of the longstanding
open problem in group testing. Bshouty proved in \cite{B12,B13} that a lower
bound of $d^2 \log n/\omega(1)$ implies that for a power of prime $q = O(d)$ one cannot
simulate a black-box multiplication of $d$ elements in the finite field $F_{q^t}$ with
$O(dt)$ black-box multiplications in $F_q$. This is one of the hard problems in algebraic
complexity.

For a randomized non-adaptive learning algorithm, just randomly choose $t=O(d(\log n+\log (1/\delta)))$
assignments $a^{(1)},\ldots,a^{(t)}$ in $\{0,1\}^n$ where each $a^{(i)}_j$ is one
with probability $1/d$ and zero with probability $1-1/d$. Define $Z=\{x_1,\ldots,x_n\}$.
Then for each assignment
that satisfies $f(a^{(i)})=0$ remove all $x_j$ from $Z$ for which $a^{(i)}_j=1$.
It is easy to show that with probability at least $1-\delta$, the variables that remain
in $Z$ are the variables of the target.

The problem of strongly attribute learnability of MClause, which is equivalent to the problem
of group testing when $d$ is not known to the learner, was studied by
Damaschke and Muhammad,~\cite{DM10,DM10b}. They
show that for deterministic non-adaptive algorithms, determining the exact number of the relevant variables $d$ is as difficult as learning the target function. For randomized non-adaptive learning algorithms,
they gave the upper bound of $O(\log n)$ to approximate $d$
and the lower bound (with some constraints) of $\Omega(\log n)$.

\subsection{Adaptive Learning $d$-MClause}\label{GTA}

In this subsection, we present the folklore algorithm for adaptively learning the class $d$-MClause. The algorithm runs in polynomial time and has a query complexity that matches the lower bound and therefore $d$-MClause is optimally learnable.

We first give the lower bound
\begin{lemma}
Any deterministic (or even randomized) algorithm for $d$-MClause must ask at least $d\log (n/d)$ queries.
\end{lemma}
\begin{proof} Follows from Lemma~\ref{lvlb} and the fact that $|C|=\sum_{i=0}^d{n\choose i}$.\qed\end{proof}

We now give the folklore algorithm. Let $f$ be the target function. For a subset $R\subseteq [n]$, define $1_R$ the assignment that is one in the entries that are in $R$ and $0$ in the other entries.
At the first stage, the algorithm defines a set $S_{1,1}:=[n]$. At stage~$i$, the algorithm has
disjoint sets $S_{i,1}, S_{i,2}, \cdots, S_{i,j_i}\subseteq [n]$ where $f(1_{S_{i,k}})=1$ for all $k=1,\ldots,j_i$.
The algorithm at stage $i$ partitions each set $S_{i,k}$, $k=1,\ldots,j_i$ into two (almost) equal disjoint sets $S_{i,k,1}\cup S_{i,k,2}$ and asks two queries $f(1_{S_{i,k,1}})$ and $f(1_{S_{i,k,2}})$.
The sets that will survive to the following stage, $i+1$, are the sets $T$ for which $f(1_T)=1$. Those will be assigned to $S_{i+1,1},\ldots,S_{i+1,j_{i+1}}$. The algorithm stops when the sizes of those sets are~$1$. Then each $S_{i+1,k}$ will be holding an index of a variable in the target.

Obviously, throughout the algorithm we have $j_i\le d$ and $|S_{i,k}|\le \lceil n/2^{i-1}\rceil$ for all $k=1,2,\ldots,j_i$. The algorithm has at most $\log n$
stages. At each stage, it asks at most $2d$ queries, and therefore,
the total number of queries is $$2d\log n=O(d\log n).$$
A more precise analysis gives the upper bound $2d\log (n/d)+O(d).$

See also the algorithms in~\cite{BHKK94,CDX14,DP94,DXSC94,ST05,UTW97} and references therein.

The above adaptive algorithm runs even if $d$ is unknown to the learner.
Therefore, the class $d$-MClause is adaptively strongly attribute-optimally learnable.
This implies
\begin{theorem}\cite{H72} The class $d$-MClause is
strongly attribute-optimally learnable with $O(d\log (n/d))$ queries.
\end{theorem}

For randomized adaptive algorithms see~\cite{DM12} and reference within.
When $d$ is unknown, Cheng~\cite{C11}
shows that there is a randomized adaptive learning algorithm
that asks $O(d\log d)$ queries and finds $d$
with probability at least $1-1/d^{\Theta(1)}$.

\subsection{Two-Round Learning}\label{TRL}

In \cite{BGV05}, De Bonis et al. shows that there is a two-round
adaptive algorithm for learning $d$-MClause that asks $O(d\log n)$ queries.
See also~\cite{CD08,EGH07}.
This is asymptotically as efficient as the best fully adaptive learning
algorithms. Therefore
$$\OPT_{2\mbox{-}{\rm RAD}}(d\mbox{-MClause})=\Theta(d\log n).$$

The algorithm uses $(2d,d+1,n)$-selector. A $(d,m,n)$-{\it selector} is a
$t\times n$ Boolean matrix such that any $d$ columns contain
at least $m$ distinct rows of Hamming weight $1$. It is known that there
is a $(2d,d+1,n)$-selector of size $t=O(d\log n)$. This follows from
the following simple probabilistic argument: randomly choose $t\times n$ matrix
where each entry is $1$ with probability $p=1/(2d)$ and $0$ with probability $1-p$.
Then show that the probability that the matrix is not a $(2d,d+1,n)$-selector
is less than one.

Given a $(2d,d+1,n)$-selector, the algorithm is as follows. Let $f=x_{i_1}\vee \cdots \vee x_{i_{d'}}$, $d'\le d$ be the target function.
At the first round, the algorithm asks $t=O(d\log n)$
queries that are the rows of a $(2d,d+1,n)$-selector. Let $X=\{x_1,\ldots,x_n\}$.
The algorithm then eliminates all the variables $x_i$ in $X$ where there is a query
$a$ for which $f(a)=0$ and $a_i=1$.
At the second round, for each variable $x_i\in X$ (that was not eliminated
in the first round) the algorithm asks the query $e^{(i)}$ where
$e^{(i)}_i=1$ and $e^{(i)}_j=0$ for all $j\not=i$. Then $x_i$ in the
target if and only if $f(e^{(i)})=1$.

Now we show that the number of variables that are
not eliminated in the first round is at most $2d-1$. Suppose
for the contrary that there
are $2d$ variables $X'=\{x_{j_1},\ldots,x_{j_{2d}}\}$
that are not eliminated in the first round. By the same argument as in Subsection~\ref{GTNA},
$\{x_{i_1}, \ldots , x_{i_{d'}}\}\subset X'$. By the property of $(2d,d+1,n)$-selectors
and since $d'\le d$ there is an assignment $a$ where $a_{i_1}=\cdots =a_{i_{d'}}=0$
and $a_{j_k}=1$ for some $j_k\not\in \{i_1,\ldots,i_{d'}\}$. This implies
that $f(a)=0$ and $a_{j_k}=1$ and the variable $x_{j_k}$ was eliminated in the first round.
This is a contradiction.

Indyk shows in \cite{I02} how to construct
an explicit $(2d,d+1,n)$-selector of size $d\cdot poly(\log n)$.
This construction gives a polynomial time learning algorithm
for $d$-MClause with $d\cdot poly(\log n)$ queries.
Therefore the class $d$-MClause is two-round efficiently learnable and
two-round optimally learnable in $d$.
Cheraghchi,~\cite{C13}, used recent results in extractors to prove that $d$-MClause
is two-round almost optimally learnable in $n$. His algorithm asks $d^{1+o(1)}\log n$ queries.

\subsection{Other Related Problems}

The group testing with inhibitors (GTI) model was introduced
in~\cite{FKKM97}. In this model, in addition to positive items and regular items,
there is also a category of items called inhibitors. The inhibitors are the items that
interfere with the test by hiding the presence of positive items. As a consequence,
a test yields a positive feedback if and only if the tested pool contains one or more
positives and no inhibitors. This problem is equivalent to learning functions
of the form $$(x_{i_1}\vee \cdots\vee x_{i_r})\wedge \overline{x_{j_1}\vee \cdots \vee x_{j_s}}.$$ This problem is studied in~\cite{BV98,CH08,DGV05,FKKM97,HC07,HL03}.

See other related problems in~\cite{Cic13,ND00} and references therein.

\section{Learning $s$-Term $r$-Monotone DNF}\label{srMDNF}
Consider the class $s$-term $r$-MDNF. That is, the class of monotone
DNF with $s$ monotone terms (monomials) where each term is of size at most $r$.
Torney,~\cite{T99}, first introduced the problem and gave some applications in molecular biology.
In this section, we present some results known from the literature for learning this class.

\subsection{Learning a Hypergraph and its Applications}

A hypergraph is $H=(V,E)$ where $V$ is the set of vertices, and $E\subseteq 2^V$
is the set of edges. The {\it dimension} or {\it rank} of the hypergraph $H$ is the cardinality
of the largest set in $E$.
A hypergraph
is called Sperner hypergraph if no edge is a
subset of another.
For a set $S\subseteq V$,
the {\it edge-detecting queries} $Q_H(S)$ is answered ``Yes'' or ``No'', indicating whether $S$ contains
all the vertices of at least one edge of $H$.
Learning the class $s$-term $r$-MDNF is equivalent to learning a Sperner hidden
hypergraph of dimension at most $r$ with at most $s$ edges using edge-detecting queries~\cite{AC06}.

This problem has many applications in chemical reactions, molecular biology, and genome sequencing.
In chemical reactions, we are given a set of chemicals,
some of which react and some which do not. When multiple chemicals are combined in one
test tube, a reaction is detectable if and only if at least one set of the chemicals in the tube reacts.
The goal is to identify which sets react using as few experiments as possible.

See \cite{AA05,ABKRS04,ABM14,AC06,AC08,BAK01,BGK05,CCFS11,CFS14,CH08,CLY13,DH06,DM10,DVMT02,GHTW06,GK98,MP04,MRY04,RS07,T99} for more details on the problem, learnability of subclasses of $s$-term $r$-MDNF
and other applications. This problem is also called, ``sets of positive subsets''~\cite{T99}
``complex group testing''~\cite{DH06,MP04} and ``group testing in hypergraph''~\cite{GHTW06}.

In all of the above applications,
the size of the terms $r$ is much smaller than the number of terms $s$
and both are much smaller than the number of vertices~$n$. Therefore, all the results in the literature, except~\cite{ABM15}, assumes that $r,s\le \sqrt{n}$,
although they do not mention this constraint explicitly. For ease of the presentation of the results,
we will also adopt this constraint throughout this section.

\subsection{Cover Free Families}\label{CVF}

One of the tools used in the literature for learning $s$-term $r$-MDNF
is cover-free families (CFF).
A $(n,(s,r))$-{\it cover free
family} ($(n,(s,r))$-CFF), \cite{KS64}, is a set $A\subseteq \{0,1\}^n$ such that for every $1\le
i_1< i_2<\cdots < i_d\le n$ where $d = s +r$ and every $J
\subseteq [d]$ of size $|J|=s$ there is $a\in A$ such that
$a_{i_k} = 0$ for all $k \in J$ and $a_{i_j} = 1$ for all $j
\not\in J$. Denote by $N(n,(s,r))$ the minimum size of such set.
The lower bound in
\cite{DVPS14,MW04,SWZ00} is\begin{eqnarray}\label{LBCFF}
N(n,(s,r))\ge \Omega\left(N(s,r)\cdot\log n\right)
\end{eqnarray}
where
\begin{eqnarray}\label{Nsr2}
N(s,r)=\frac{s+r}{\log{s+r\choose r}}{s+r\choose r}.
\end{eqnarray}

It is known, \cite{ABM14}, that a set of
\begin{eqnarray}\label{rand}
m&=&O\left(\min(s,r)^{1.5}\left(\log \left(\frac{s+r}{\min(s,r)}\right)\right)
\left (N(s,r)\cdot\log n+\frac{N(s,r)}{s+r}\log\frac{1}{\delta}\right)\right)\nonumber\\
&=& N(s,r)^{1+o(1)}(\log n+\log(1/\delta))
\end{eqnarray}
random vectors $a\in \{0,1\}^n$, where each $a_j$ is $1$ with probability $r/(s+r)$, is a $(n,(s,r))$-CFF with probability at least $1-\delta$.

It follows from \cite{B12,B14b,BG15,FLS14} that there is
a polynomial time (in the size of the CFF)
deterministic construction of $(n,(s,r))$-CFF
of size
\begin{eqnarray}\label{recent}
{N(s,r)}^{1+o(1)}\log n\end{eqnarray} where the $o(1)$ is with respect to $\min(r,s)$. When $r=o(s)$, the construction can be done in linear time~\cite{B14b,BG15}.

\subsection{Non-Adaptive Learning $s$-Term $r$-MDNF}

In this section, we give a non-adaptive learning algorithm for
the class of $s$-term $r$-MDNF.

We first give a lower bound
\begin{theorem}\cite{ABM15,DH06} Let $n\ge r+s$.
Any equivalent test set $A\subseteq \{0,1\}^n$ for $s$-term $r$-MDNF
is $(n,(s,r-1))$-CFF and $(n,(s-1,r))$-CFF. Therefore,
Any non-adaptive algorithm for learning
$s$-term $r$-MDNF must ask at least
$$N(n,(s-1,r))+N(n,(s,r-1))= \Omega\left(N(s,r)\log n\right)$$ queries.

In particular,
when $r$ is constant, the number of queries is at least
$$\OPTNA(\mbox{$s$-term $r$-MDNF})= \Omega\left(\frac{s^{r+1}}
{\log s}\log n\right).$$
\end{theorem}
\begin{proof} Consider any distinct $1\le i_1,\cdots,i_{r+s-1}\le n$.
To be able to distinguish between the two functions
$f_1=(x_{i_1}\cdots x_{i_r})\vee x_{i_{r+1}}\vee \cdots \vee x_{i_{r+s-1}}$
and $f_2=(x_{i_1}\cdots x_{i_{r-1}})\vee x_{i_{r+1}}\vee \cdots \vee x_{i_{r+s-1}}$
we must have an assignment $a$ that satisfies $a_{i_1}=\cdots=a_{i_{r-1}}=1$
and $a_{i_r}=\cdots=a_{i_{r+s-1}}=0$. Therefore $A$ is $(n,(s,r-1))$-CFF.

To be able to distinguish between the two functions
$g_1=(x_{i_1}\cdots x_{i_r})\vee x_{i_{r+1}}\vee \cdots \vee x_{i_{r+s-1}}$
and $g_2=x_{i_{r+1}}\vee \cdots \vee x_{i_{r+s-1}}$
we must have an assignment $a$ that satisfies $a_{i_1}=\cdots=a_{i_{r}}=1$
and $a_{i_{r+1}}=\cdots=a_{i_{r+s-1}}=0$. Therefore $A$ is $(n,(s-1,r))$-CFF.\qed
\end{proof}

We now give a simple upper bound
\begin{theorem} Any $(n,(s,r))$-CFF
$A$ is an equivalent test set for $s$-term $r$-MDNF.
Therefore, there is a non-adaptive learning algorithm
for $s$-term $r$-MDNF that asks
$$N(n,(s,r))= O\left(\sqrt{r}{s+r\choose r}(s+r)\log n\right)$$ queries.
In particular,
when $r$ is constant,
$$\OPTNA(\mbox{$s$-term $r$-MDNF})\le N(n,(s,r))=O\left({s^{r+1}}
\log n\right)$$
\end{theorem}
\begin{proof} Let $A$ be a $(n,(s,r))$-CFF.
Let $f_1,f_2$ be any two non-equivalent $s$-term $r$-MDNF.
Suppose $f_1=M_{1,1}\vee M_{1,2}\vee \cdots \vee M_{1,s_1}$ and
$f_2=M_{2,1}\vee M_{2,2}\vee \cdots \vee M_{2,s_2}$ where $s_1,s_2\le s$.
Let $b\in \{0,1\}^n$ be an assignment such that (w.l.o.g.) $f_1(b)=1$
and $f_2(b)=0$. Then $M_{2,i}(b)=0$ for all $i=1,\ldots,s_2$ and $M_{1,j}(b)=1$
for some $1\le j\le s_1$. Let $M_{1,j}=x_{k_1}x_{k_2}\cdots x_{k_{r'}}$, $r'\le r$.
Then for every $M_{2,i}$, $i=1,\ldots,s_2$, there is a variable
$x_{j_i}$ in $M_{2,i}$ where $b_{j_i}=0$ and for all the variables
$x_{k_1},x_{k_2},\ldots,x_{k_{r'}}$  in $M_{1,j}$ we have $b_{k_1}=\cdots=b_{k_{r'}}=1$.

Now take $a\in A$ such that $a_{j_i}=0$, $i=1,\ldots,s_2$ and $a_{k_1}=a_{k_2}=\cdots=a_{k_{r'}}=1$.
Such $a$ exists since $A$ is $(n,(s,r))$-CFF.
Then we have $f_1(a)=1$ and $f_2(a)=0$.
This completes the proof.\qed
\end{proof}

The first explicit non-adaptive learning algorithm for $s$-term $r$-MDNF
was given by Gao et al., \cite{GHTW06}.
They show that this class can be learned with $(n,(s,r))$-CFF.
Given such a $(n,(s,r))$-CFF, the algorithm simply takes all the monomials $M$ of size at most $r$
that satisfy $(\forall a\in A)(M(a)=1\Rightarrow f(a)=1)$. It
is easy to see that the disjunction
of all such monomials is equivalent to the target function.
Assuming a set of $(n,(s,r))$-CFF of size $N$ can be constructed in time $T$,
the above algorithm learns $s$-term $r$-MDNF with $N$ queries in time $O({n\choose r}N+T)$.
This with~(\ref{recent}) gives

\begin{theorem}\label{nafmdnf}
There is a non-adaptive learning algorithm
for $s$-term $r$-MDNF that asks
$$Q={N(s,r)}^{1+o(1)}\log n$$ queries
and runs in $poly({n\choose r},Q)$ time.

When $r$ is constant, the algorithm asks
$$Q'=O(s^{r+1}\log n)$$ queries and runs in $poly(n)$ time.

In particular, for constant $r$, the class $s$-term $r$-MDNF is
non-adaptively almost optimally learnable
and optimally learnable in $n$.
\end{theorem}

When $r=\omega(1)$ we can use Lemma~\ref{DNtoDN} to prove

\begin{theorem} Let $r=\omega(1)$.
There is a non-adaptive learning algorithm
for $s$-term $r$-MDNF that asks
$$Q={N(s,r)}^{1+o(1)}\log n$$ queries
and runs in $poly(n,{2(rs)^2\choose r})$ time.

In particular, for $r\le s^{c}$ for some $c<1$, the class $s$-term $r$-MDNF is
non-adaptively almost optimally learnable
and optimally learnable in $n$.
\end{theorem}
\begin{proof} Follows from Lemma~\ref{DNtoDN} and Theorem~\ref{nafmdnf}
and the fact that any $s$-term $r$-MDNF has at most $sr$ relevant variables.\qed
\end{proof}

One can now use~(\ref{rand}) in a straightforward manner
to get a randomized non-adaptive algorithm with
better time and query complexity. Recently, Abasi et al.~\cite{ABM15},
gave an almost optimal learning algorithm for all $n,r$ and $s$.

\subsection{Adaptive Learning $s$-Term $r$-MDNF}

In this section, we give results on adaptive algorithms for learning
$s$-term~$r$-MDNF.

Adaptive algorithms for learning
$s$-term $r$-MDNF is studied in \cite{AC06,AC08} and~\cite{ABM14}.
The information theoretic lower bound for this class is $rs\log n$.
Angluin and Chen gave in~\cite{AC08} the lower bound $\Omega((2s/r)^{r/2}+rs\log n)$
when $s> r$ and Abasi et al. gave in~\cite{ABM14} the lower bound $\Omega(({r}/{s})^{s-1}+rs\log n)$ when $s\le r$. Angluin and Chen gave a
polynomial time adaptive algorithm for learning $s$-term $2$-MDNF that
asks $O(s\log n)$ queries. Therefore, the class $s$-term $2$-MDNF is
adaptively optimally learnable. In \cite{ABM14} Abasi et al. gave a polynomial
time learning algorithm for $s$-term $r$-MDNF that asks $rs\log n +({r}/{s})^{s+o(s)}$ queries when $r>s$ and $rs\log n +s^{r/2+o(r)}$ queries
when $r\le s$. They also gave some randomized algorithms.

The following table summarizes the latest results: Det. and Rand. stand for deterministic algorithm and randomized algorithm, respectively.
\begin{center}
\begin{tabular}{|l|c|c|l|l|}
&& &  Rand./&\\
$r,s$ & Lower Bound & Ref. & Det.& Upper Bound\\
\hline
\hline
$r=2$ & $s\log n$ & \cite{AC08} & Det. & $s\log n$\\
\hline
$r>s$&$rs\log n+\left(\frac{r}{s}\right)^{s-1}$ & \cite{ABM14} & Det. &$rs\log n +\left(\frac{r}{s}\right)^{s+o(s)}$\\
\cline{3-5}
&& \cite{ABM14} & Rand. & $rs\log n +(\log r)\sqrt{s}e^s\left(\frac{r}{s}+1\right)^s$\\
\hline
$r\le s$ & $rs\log n+\left(\frac{2s}{r}\right)^{r/2}$ & \cite{ABM14,AC08} & Det. &$rs\log n +s^{r/2+o(r)}$\\
\cline{3-5}
 &  & \cite{ABM14,AC08} & Rand. &$rs\log n +\sqrt{r}(3e)^{r}(\log s) s^{r/2+1}$\\
\hline
\end{tabular}
\end{center}

\subsection{Learning Subclasses of $s$-term $r$-MDNF}

Learning subclasses of graphs and hypergraphs from edge-detecting queries received
considerable attention in the literature due to its diverse applications~\cite{AC06,AC08,BAK01,CCFS11}.
This is equivalent to
learning subclasses of $s$-term $2$-MDNF and $s$-term $r$-MDNF,
respectively. Subclasses include Graphs of bounded degree~\cite{AC08},
Hamiltonian cycles~\cite{CCFS11,GK97,GK98}, matchings~\cite{ABKRS04,BAK01,CCFS11}, stars~\cite{AA05,CCFS11},
cliques~\cite{AA05,CCFS11}, families of graphs that are closed under isomorphism~\cite{ABKRS04}
$r$-uniform hypergraph and almost uniform hypergraph~\cite{AC08}.
The class of Read-Once $2$-MDNF is equivalent to learning matchings~\cite{ABKRS04}.

\section{Learning Decision Tree}\label{Decision}
In this section we study the learnability of the class of Depth $d$-DT ($\DT_d$), i.e., the class of
all decision trees of depth at most $d$.

\subsection{Bounds on $\OPT(\DT_d)$}

We say that a set of assignments $A\subseteq \{0,1\}^n$ is $(n,d)$-{\it universal set} if for every $1\le i_1<i_2<\cdots<i_d\le n$ and every $b\in \{0,1\}^d$ there is an $a\in A$ such that $a_{i_j}=b_j$ for all $j=1,\ldots,d$.

It is known that any $(n,d)$-universal set is of size  $\Omega(2^d\log n)$~\cite{KS72,SB}.
The probabilistic method with the union bound gives the upper bound $O(d2^d\log
n)$. The best known polynomial time, $poly(n,2^d)$,
construction  gives an $(n,d)$-universal set
of size $O(2^{d+\log^2 d}\log n)$ \cite{NSS95}. It is easy to show that a random uniform set $A\subseteq \{0,1\}^n$ of size $O(d2^d\log n)$ is $(n,d)$-universal set with probability at least $1/n^{O(d)}$.
For $d\le \log n/\log\log n$,
an $(n,d)$-universal set of size $d^{O(1)}2^d\log n$ can be constructed in polynomial
time~\cite{NSS95}.

We now prove
\begin{theorem} A set of assignments $A\subseteq \{0,1\}^n$ is a hitting set for $\DT_d$ if and only if $A$ is an $(n,d)$-universal set. In particular
$$\Omega(2^d\log n)\le \OPT^{\HS}(\DT_d)\le O(d2^d\log n).$$
\end{theorem}
\begin{proof} Let $A$ be an $(n,d)$-universal set. Let $f\not\equiv 0$ be a decision tree of depth at most $d$. Consider a path $p$ from the root of $f$
to a leaf that is labeled with $1$. Let $p=v_1\stackrel{{\scriptscriptstyle {\xi_1}}}{\to} v_2\stackrel{{\scriptscriptstyle {\xi_2}}}{\to} \cdots\stackrel{{\scriptscriptstyle {\xi_{d'-1}}}}{\to} v_{d'}
 \stackrel{{\scriptscriptstyle {\xi_{d'}}}}{\to}v_{d'+1}$, $d'\le d$
where $v_{d'+1}$ is the leaf that is labeled with $1$, $v_i$ is labeled with $x_{j_i}$ and the edge $v_j\to^{\xi_j} v_{j+1}$ is labeled with $\xi_j\in\{0,1\}$. If $A$ is $(n,d)$-universal set then there is $a\in A$ such that $a_{j_i}=\xi_i$ for all $i=1,\ldots,d'$ and $f(a)=1$. Therefore, $A$ is a hitting set for $DT_d$.

The other direction follows from the fact that any term $T=x_{j_1}^{\xi_1}\cdots x_{j_d}^{\xi_d}$ where $1\le j_1<\cdots<j_d\le n$ and $\xi_i\in\{0,1\}$ is a decision tree of depth $d$.
Recall that $x^\xi=x$ if $\xi=1$ and $x^\xi=\bar x$ if $\xi=0$.
To hit $T$ we need an assignment $a$ that satisfies $a_{j_1}=\xi_1,\ldots,a_{j_d}=\xi_d$.\qed
\end{proof}

For equivalent test set, adaptive and non-adaptive learning, we prove
\begin{theorem}\label{DTHT}
We have
$$\Omega(2^d\log n)\le \OPT_\A(C)\le \OPT_\NA(C)\le O(d2^{2d}\log n).$$
\end{theorem}
\begin{proof} Let ${\cal A}$ be an adaptive algorithm that learns DT${}_d$.
We run ${\cal A}$ and answer $0$ for all the membership queries.
The algorithm must stop and return the function $0$ as the target function.
Let $A$ be the set of all the assignments asked in the membership queries.
If $A$ is not $(n,d)$-universal set, then there is $1\le j_1<\cdots<j_d\le n$ and $(\xi_1,\ldots,\xi_d)\in \{0,1\}^d$ such that $(a_{j_1},\ldots,a_{j_d})\not=(\xi_1,\ldots,\xi_d)$.
Then the term $T=x_{j_1}^{\xi_1}\cdots x_{j_d}^{\xi_d}\in \DT_d$
is zero on all the
assignments of $A$, and we get a contradiction.

We now prove that $\DT_d\oplus\DT_d\subset \DT_{2d}$ and then the other results follow from {\it 3} in Lemma~\ref{xor}.

For any two functions $f,g\in \DT_d$ and the corresponding decision trees
$T(f)$ and $T(g)$, one can construct a tree for $f\oplus g$ from $T(f)$ and $T(g)$ as follows. First $g+1$ is equivalent to the tree $T(g+1)$ that is the same as $T(g)$ where the labels in the leafs are flipped from $0$ to $1$ and from $1$ to $0$. Now in the tree $T(f)$, replace each leaf that is labeled with $0$ with the tree $T(g)$ and each leaf that is labeled with $1$ with the tree $T(g+1)$. It is easy to show that this tree computes $f\oplus g$, and its depth is at most $2d$.\qed
\end{proof}

\subsection{Adaptive Learning Decision Tree}
The adaptive learnability of decision tree
of depth $d$ follows from many papers~\cite{B12,B13,BJT04,BM02,F07,J97,KM93,SS96}.
One of the powerful techniques used in the literature is the discrete Fourier transform DFT.
In DFT, one regards
the Boolean function $f:\{0,1\}^n\to \{0,1\}$ as a real function $f:\{0,1\}^n\to \{+1,-1\}$
in $\Re^{\{0,1\}^n}$,
represent it as a linear combination of orthonormal basis functions,
and then learns the coefficients.

In \cite{KM93}, Kushilevitz and Mansour used this technique for
learning the class of decision trees as follows. Consider the set ${\cal
A}=\{\chi_a:\{0,1\}^n\to\{-1,+1\} | a\in \{0,1\}^n\}$ where
$\chi_a(x)=\prod_{i=1}^n (-1)^{a_ix_i}$.
It is easy to see that ${\cal
A}$ is an orthonormal basis for the set of functions $f:\{0,1\}^n\to
\Re$. Therefore, every function $f:\{0,1\}^n\to \Re$ can be
represented as
$$f=\sum_{a\subseteq \{0,1\}^n}\hat f(a) \chi_a$$ where $\hat f(a)\in \Re$.
This representation is called the Fourier representation of $f$
and $\hat f(a)$ is called the {\it Fourier coefficient of} $\chi_a$.
It is easy to see that the Fourier coefficients of $\chi_a$ is $\hat
f(a)=\E_x[f(x)\chi_a(x)]$ where $x\in\{0,1\}^n$, and
the expectation is over the uniform distribution on $\{0,1\}^n$.
So every coefficient can be estimated using Chernoff
bound. It remains to show that for a decision tree of depth $d$ the
number of nonzero Fourier coefficients $\hat f(a)$ is small, and they can be
found exactly and efficiently\footnote{An efficient learning algorithm for
decision tree of depth $d$ is one that asks $poly(2^d,\log
n)$ queries and runs in $poly(2^d,n)$ time. See Section~\ref{VE}.}.
We demonstrate the algorithm with the help of the following simple example.
$$ $$
$$ $$
$$ $$
\begin{figure}
\centering
\includegraphics[trim = 2.5cm 5cm 0 5cm,width=1\textwidth]{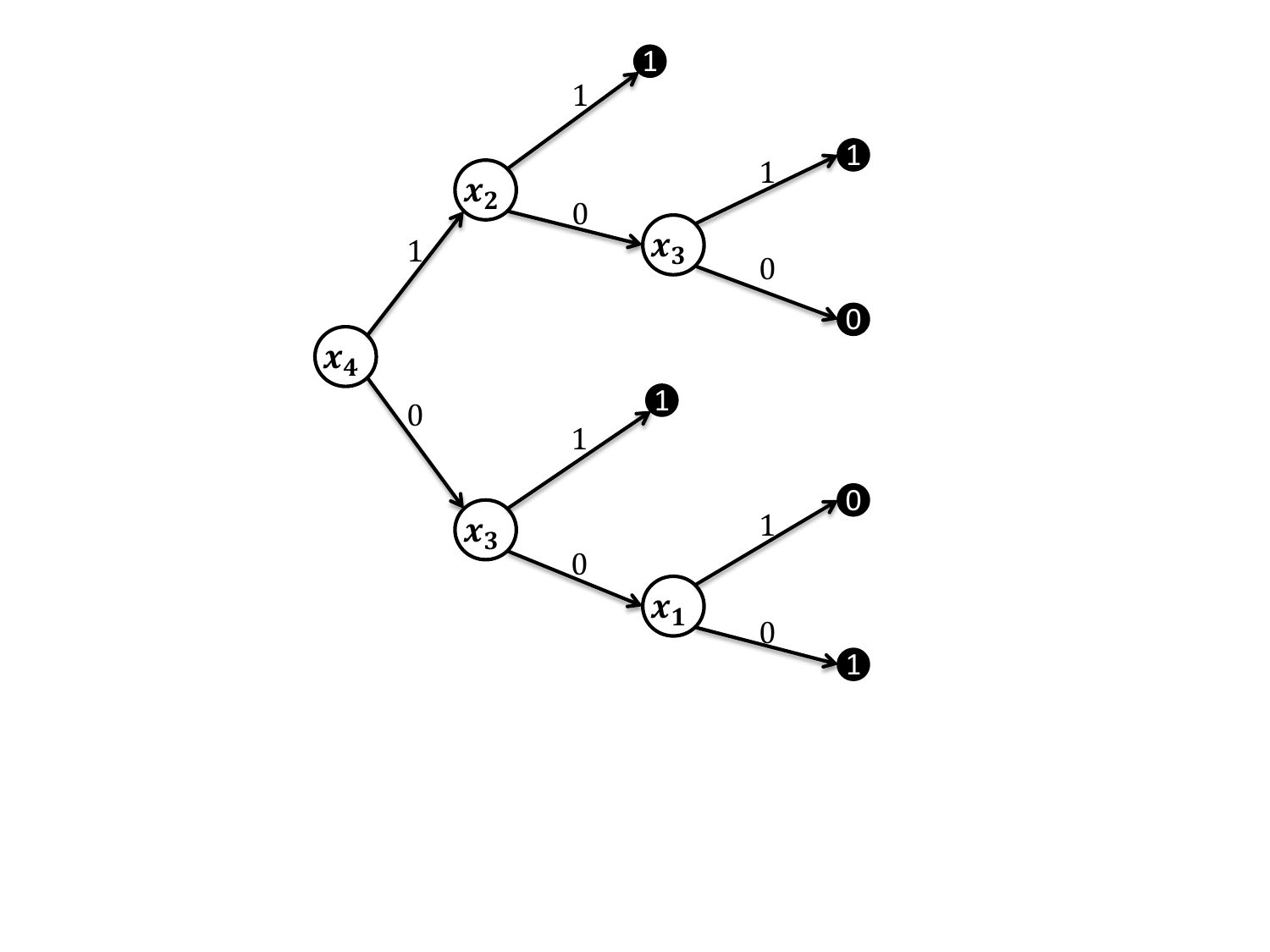}
\caption{The decision tree of $f=x_4(x_2+\bar x_2
x_3)+\bar x_4(x_3+\bar x_3\bar x_1)$. The tree is of depth $3$ and size $6$.}
\label{DTree02}
\end{figure}

Consider the decision tree in Figure~\ref{DTree02}: $f=x_4(x_2+\bar x_2
x_3)+\bar x_4(x_3+\bar x_3\bar x_1)$. In this example, the depth of $f$ is $3$, and $f$ is a
sum of terms of size $3$. First notice that since the
terms are disjoint (no two terms are equal to $1$ for the same
assignment), the ``+'' operation can be replaced by the arithmetic
``+'' operation in $\Re$. To change the values of the function to $+1,-1$ values,
we take $g=1-2f$.
In general, every decision tree of size
$s$ and depth $d$ can be written as a sum (in $\Re$) of $s$ terms
of size at most~$d$. Now take any term, say $t=x_4\bar x_2 x_3$. Over the
real numbers $\Re$, we can express $x$ as $(1-(-1)^x)/2$ and $\bar
x$ as $(1+(-1)^x)/2$. Then the term $t$ can be expressed as
\begin{eqnarray*}
t&=& \frac{1-(-1)^{x_4}}{2}\cdot \frac{1+(-1)^{x_2}}{2}\cdot \frac{1-(-1)^{x_3}}{2}\\
&=& \frac{\chi_{(0000)}+\chi_{(0100)}-\chi_{(0010)}-\chi_{(0001)}-\chi_{(0110)}
-\chi_{(0101)}+\chi_{(0011)}+\chi_{(0111)}}{8}.
\end{eqnarray*}
In general, every term of size $d$ has a Fourier representation
that contains $2^d$ non-zero coefficients $\hat f(a)$ each is $\pm 1/2^d$
and $wt(a)\le d$  where $wt(a)$ is the Hamming weight of $a$, i.e, the number of ones in $a$.
Therefore, every decision tree of size $s$ and
depth $d$ has a Fourier representation that contains at most $s2^d+1\le
2^{2d}+1$ non-zero Fourier coefficients, each has one of the
values in $\{\pm k/2^d\ |\ k\in [s2^d]\}\cup \{0\}$.
In fact, using Parseval's identity, one can prove that $\sum_{a\in \{0,1\}^n} \hat f(a)^2=1$
and, therefore, the
Fourier coefficients have values from $\{\pm k/2^d\ |\ k\in [2^d]\}\cup \{0\}$. Also each
non-zero coefficient $\hat f(a)$ satisfies $wt(a)\le d$. Now using
Chernoff bound, for each assignment $a$ of weight at most $d$, one can exactly find each
coefficient $\hat f(a)=\E_x[f(x)\chi_a(x)]$ with
$poly(s,2^d)$ queries. The problem with this algorithm is
that since the number of assignments $a$ that satisfies $wt(a)\le d$ is $n^{O(d)}$
the time complexity is exponential $n^{O(d)}$.

Kushilevitz and Mansour in \cite{KM93}
and Goldreich and Levin in \cite{GL89} gave an adaptive algorithm
that finds the non-zero coefficients in $poly(2^d,n)$ time and
queries. Kushilevitz and Mansour algorithm (KM-algorithm) is based on the fact
that for any $\alpha\in \{0,1\}^r$ we have
$$F_\alpha:=\underset{y,z\in \{0,1\}^r, x\in \{0,1\}^{n-r}}{\E} \left[ f(yx)
f(zx)\chi_\alpha(y)\chi_\alpha(z)\right]=\sum_{x \in \{0,1\}^{n-r}} \hat f(\alpha x)^2,$$
where for $y=(y_1,\ldots,y_r)$ and $x=(x_1,\ldots,x_s)$,
$yx=(y_1,\ldots,y_r,x_1,\ldots,x_s)$.
Notice that $F_\alpha\le 1$ can be computed exactly (with high probability) with Chernoff bound.
Now KM-algorithm uses divide and conquer technique with the above identity to
find the non-zero coefficients in $poly(2^d,n)$ time. Notice that
$$F_{()}= \sum_{x \in \{0,1\}^{n}} \hat f(x)^2,\ \ F_{(0)}=\sum_{x \in \{0,1\}^{n-1}} \hat f(0 x)^2,\ \ F_{(1)}=\sum_{x \in \{0,1\}^{n-1}} \hat f(1 x)^2$$ and therefore $F_{()}=F_{(0)}+F_{(1)}$.
The algorithm first computes $F_{(0)}$ and $F_{(1)}$
and lets $T_1=\{\xi\in \{0,1\}|F_{(\xi)}\not=0\}$.
At some stage, it holds a set $T_r=\{\alpha^{(1)},\ldots,\alpha^{(t)}\}\subseteq \{0,1\}^r$ where $F_\alpha\not=0$
for all $\alpha\in T_r$ and $F_\alpha=0$
for all $\alpha\not\in T_r$. Now for each $\alpha\in T_r$ it computes $F_{(\alpha,0)}$ and $F_{(\alpha,1)}$. Since $0\not= F_\alpha =F_{(\alpha,0)}+F_{(\alpha,1)}$, at least one of them is not zero.
Then it defines $T_{r+1}=\{(\alpha,\xi)\ |\ F_{(\alpha,\xi)}\not=0\}$. Since the number
of the non-zero Fourier coefficients of a decision tree of depth at most $d$ is less than $2^{2d}+1$, the number of elements in
$T_{r+1}$ is less than $2^{2d}+1$. Notice that for $\alpha\in \{0,1\}^n$
\begin{eqnarray*}
F_\alpha&=&\underset{y,z\in \{0,1\}^n}{\E} \left[ f(y)
f(z)\chi_\alpha(y)\chi_\alpha(z)\right]\\
&=&\underset{y\in \{0,1\}^n}{\E} \left[ f(y)
\chi_\alpha(y)\right]\underset{z\in \{0,1\}^n}{\E} \left[
f(z)\chi_\alpha(z)\right]=\hat   f(\alpha)^2.
\end{eqnarray*}
and therefore, $T_n$ contains all the assignments $\alpha\in \{0,1\}^n$
for which $\hat f(\alpha)\not=0$. It is easy to see that this algorithm
runs in time $poly(2^d,n)$. Therefore

\begin{theorem} \cite{KM93} There is an adaptive Monte Carlo learning algorithm that
learns $\DT_d$ in $poly(2^d,n)$ time and $O(2^{6d}\cdot n)$ membership queries.
\end{theorem}

Kushilevitz and Mansour use a derandomization technique to change the algorithm to deterministic.
They prove
\begin{theorem}\label{DTD} \cite{KM93} There is an adaptive deterministic learning algorithm that
learns $\DT_d$ in $poly(2^d,n)$ time and $O(2^{10d}\cdot n\log^2 n)$ membership queries.
\end{theorem}

By Lemma~\ref{BHL0}, we have

\begin{theorem} There is an adaptive deterministic learning algorithm that
learns $\DT_d$ in $poly(2^d,n)$ time and $O(2^{13d+o(d)}\cdot \log n)$ membership queries.

In particular, $\DT_d$ is efficiently adaptively learnable and optimally adaptively learnable in $n$.
\end{theorem}
\begin{proof} We use Lemma~\ref{BHL0}. Since decision trees of depth $d$ have at most $2^d$ relevant variables, we can set $r=2^d$. By Theorem~\ref{DTD}, $Q(r)=O(d^22^{11d})$.
An $(n,d)$-universal set is a constant testing set for $\DT_d$. See the proof of Theorem~\ref{DTHT}.
The best known polynomial time, $poly(n,2^d)$,
construction  gives an $(n,d)$-universal set
of size $O(2^{d+o(d)}\log n)$ \cite{NSS95}. Therefore $P(n,r)=O(2^{d+o(d)}\log n)$.
Then the reduction in Lemma~\ref{BHL0} gives a polynomial time adaptive learning algorithm that asks
$$O(r\cdot Q(r)P(n,r)+r\log n)=2^{13d+o(d)}\log n$$ membership queries.\qed
\end{proof}

See other randomized algorithms in \cite{BM02,SS96} that
use different techniques. The algorithm in \cite{SS96} uses membership and
equivalence queries, and it is easy to see that every equivalence query can be simulated
by randomized membership queries.

\subsection{Non-Adaptive Learning Decision Tree}
In this subsection, we give a sketch of the results in \cite{BJT04,GL89}
    and then of \cite{F07} that gave the first polynomial time Monte Carlo non-adaptive learning algorithm for~$\DT_d$.

The following is the result of Hofmeister in~\cite{H99}
\begin{lemma}~\cite{H99}\label{H99} There is a polynomial time deterministic non-adaptive algorithm for $C=\{\chi_a(x)|wt(a)\le d\}$
that asks $O(d\log n)$ membership queries.

In particular, there is a set of assignments $B=\{b^{(1)},\ldots,b^{(t)}\}$ of size $t=O(d\log n)$ that can be constructed in polynomial time and an algorithm ${\cal A}$ such that: Given $\chi_a(b^{(1)}),\ldots,\chi_a(b^{(t)})$ for some $a$ of weight at most $d$, the algorithm ${\cal A}$ finds the assignment $a$ in polynomial time.
\end{lemma}

The main idea of the learning algorithm of $\DT_d$ is to use pairwise independent
assignments for estimating $\E_x[f(x)\chi_a(x)]$, rather than totally independent assignments.
Since pairwise independent assignments can be generated with a small number of random bits, the problem
is reduced to finding the Fourier coefficient of a function that depends on a small number of variable.
Using those coefficients one can recover the assignments $a$ with large Fourier coefficients $\hat f(a)$ in $f$. We now give a sketch of the algorithm and its correctness.

It is easy to see that for the function $f_b(x):=f(x+b)$ we have
\begin{eqnarray}\label{zuzu}
\hat f_b(a)=\chi_a(b)\hat f(a).
\end{eqnarray} Therefore
$$\chi_a(b)\cdot \sign(\hat f(a))=\sign(\hat f_b(a))=\sign\left(\E_x[f(x+b)\chi_a(x)]\right)$$ where $\sign(x)=1$ if $x\ge 0$ and $=-1$ otherwise. Therefore
$$\chi_a(b)=\sign(\hat f(a))\cdot \sign\left(\E_x[f_b(x)\chi_a(x)]\right).$$
Assuming we know $\sign(\hat f(a))$, to compute $\chi_a(b)$ for some $b$ we only need to
know the sign of $\E_x[f_b(x)\chi_a(x)]$. Notice that $|\E_x[f_b(x)\chi_a(x)]|=|\hat f_b(a)|=|\hat f(a)|$.
Now assuming $|\E_x[f_b(x)\chi_a(x)]|=|\hat f(a)|$
is not zero (and therefore $\ge 1/2^d$), to compute the sign of $\E_x[f_b(x)\chi_a(x)]$ it is enough to use Bienayme-Chebyshev
bound rather than Chernoff bound. That is, to estimate $f_b(x)\chi_a(x)$ on pairwise independent
assignments rather than totally independent assignments. To generate such assignments,
consider a random uniform $k\times n$ matrix $R$ over the binary field $F_2=\{0,1\}$ where $k=\log m$.
We will determine $m$ later. Then the assignments in the set $\{pR\ |\ p\in \{0,1\}^{k}\}$ are pairwise
independent. Combining all the above ideas with Bienayme-Chebyshev
bound we prove
\begin{lemma}\label{fco} With probability at least $1-4/(m(\hat f(a))^2)$ we have
$$\chi_a(b)= \sign(\hat f(a)) \cdot \sign(\hat f_{R,b}(aR^T))$$
where $f_{R,b}(x_1,\ldots,x_k)=f(xR^T+b)$.
\end{lemma}
\begin{proof}
Let $\{z_1,\ldots,z_m\}=\{pR|p\in\{0,1\}^k\}$. Consider a random uniform $x\in\{0,1\}^n$ and the random variable $X(x)=f(x)\chi_a(x)$. We have $\mu=\E[X]=\hat f(a)$ and $\sigma=\Var[X]=\E[X^2]-\E[X]^2=1-\hat f(a)^2$. Since $X_i=X(z_i)$ are pairwise independent, by Bienayme-Chebyshev
bound we get
$$\Pr\left[\left|\frac{\sum_{i=1}^m X_i}{m}-\hat f(a)\right|\ge \frac{|\hat f(a)|}{2}\right]\le \frac{\sigma}{m \left(\frac{|\hat f(a)|}{2}\right)^2}\le \frac{4}{m\hat f(a)^2}.$$
Therefore, with probability at least $1-2/(m\hat f(a)^2)$ we have
$$\sign(\hat f(a))= \sign\left(\frac{\sum_{i=1}^m f(z_i)\chi_a(z_i)}{m}\right).$$
By (\ref{zuzu}), with probability at least $1-4/(m\hat f(a)^2)$ we have
$$\chi_a(b)\cdot\sign(\hat f(a))= \sign\left(\frac{\sum_{i=1}^m f_b(x_i)\chi_a(x_i)}{m}\right).$$
Since
\begin{eqnarray*}
\frac{\sum_{i=1}^m f_b(z_i)\chi_a(z_i)}{m}&=& \frac{1}{m} \sum_{p\in\{0,1\}^k} f_b(pR)\chi_a(pR)\\
&=& \frac{1}{m} \sum_{p\in\{0,1\}^k} f_{R,b}(p)\chi_{aR^T}(p)=\hat f_{R,b}(aR^T)\\
\end{eqnarray*}
The result follows.\qed
\end{proof}
Notice that $f_{R,b}$ is a function in $k=\log m$ variables. This is the key lemma.
It shows that if $\hat f(a)\not=0$, there is a positive probability that $\chi_a(b)$
can be computed (modulo the $\sign(\hat f(a))$) using the sign of some Fourier coefficient of $\hat f_{R,b}$. Since $f_{R,b}$ depends on a small number of variable and a membership to $f_{R,b}$ can be simulated by a membership to $f$ (since $f_{R,b}(a_1,\ldots,a_k)=f((a_1,\ldots,a_k)R^T+b)$), all its Fourier coefficients can be easily found.
Now if $\chi_a(b)$ is computed for all $b\in B=\{b^{(1)},\ldots,b^{(t)}\}$, where
$B$ is the set in Lemma~\ref{H99},
then $a$ can be found with Hofmeister's algorithm. In the following, we give more details.

By Lemma~\ref{fco} and using the union bound we have
\begin{lemma}\label{ooo} Let $B=\{b^{(1)},\ldots,b^{(t)}\}$ be as in Lemma~\ref{H99} and $t=|B|$. Let $m=2^{2d+3}t$.
For any $a$, if $|\hat f(a)|\ge 1/2^d$, there is $z\in \{0,1\}^{k}$  and $\xi\in\{-1,+1\}$ such that with probability
at least $1-(2^{2d+2}t)/m\ge 1/2$ we have
$$(\chi_a(b^{(1)}),\ldots,\chi_a(b^{(t)}))= \xi \cdot (\sign(\hat f_{R,b^{(1)}}(z)),\ldots,\sign(\hat f_{R,b^{(t)}}(z)))$$
where $f_{R,b^{(i)}}(x_1,\ldots,x_k)=f(xR^T+b^{(i)})$.
\end{lemma}

Now since $f_{R,b^{(i)}}$ depends on $k=\log m$ variables we can find all
its Fourier coefficients in time $poly(m)$ and $2^k=m=O(2^{2d}d\log n)$ membership queries. Therefore, in $poly(m,t)$ time and $2^kt=O(2^{2d}d^2\log^2 n)$ membership queries
we can find $$W_{\xi,z}:=\xi\cdot (\sign(\hat f_{R,b^{(1)}}(z)),\ldots,\sign(\hat f_{R,b^{(t)}}(z)))$$ for all
$z$ and $\xi$. If $\hat f(a)$ is not zero then $|\hat f(a)|\ge 1/2^d$ and then by Lemma~\ref{ooo}, with probability
at least $1/2$,
some $z$ and $\xi$ satisfies $(\chi_a(b^{(1)}),\ldots,\chi_a(b^{(t)}))=W_{\xi,z}$. Then by Lemma~\ref{H99}, $a$ can be recovered in polynomial time. So from all $W_{\xi,z}$ and using the algorithm
${\cal A}$ in Lemma~\ref{H99} we find a set of assignments $Q$ such that: if $\hat f(a)$ is not zero then
$a\in Q$ with probability at least $1/2$. This implies that on average, $Q$ contains half of the
assignments that correspond to the non-zero Fourier coefficients of $f$. The size of $Q$ is at most $2^{k+1}=O(2^{2d}d\log n)$.
Then we find the Fourier coefficient $\hat f_a =\E_x[f(x)\chi_a(x)]$ for all $a\in Q$ using Chernoff bound and the union bound with $O(2^d\log|Q|)=o(2^k)$ additional membership queries.
We can repeat the above $\log (2^{2d}/\delta)$ time to find all the non-zero Fourier coefficients of $f$ with
probability at least $1-\delta$.

Putting all the above ideas together, it follows that
\begin{lemma} \cite{F07}. There is a non-adaptive Monte Carlo learning algorithm that
learns $\DT_d$ in polynomial time and $O(d^32^{2d}\log^2n)$ membership queries.
\end{lemma}

By Lemma~\ref{DNtoDN}, we get
\begin{lemma}  There is a non-adaptive Monte Carlo learning algorithm that
learns $\DT_d$ in polynomial time and $O(d^42^{4d}\log n)$ membership queries.

In particular, $\DT_d$ is MC efficiently non-adaptively learnable and
MC optimally non-adaptively learnable in $n$.
\end{lemma}
\begin{proof} We use Lemma~\ref{DNtoDN}. Since a decision tree of depth at most $d$ contains at most $2^d$ relevant variables, we can take $r=2^d$. We take $q=2^{4d}$. By Lemma~\ref{DNtoDN}, $Q(q)=O(d^32^{2d}\log^2 q)$.
Then the number of membership queries is
$$O\left(\frac{r^2 Q(q)}{\log (q/r^2)}\log n\right)= O(d^42^{4d} \log n).$$
\qed
\end{proof}

A better query complexity can be obtained from the reduction in~\cite{BC16}. See the following Table.

The outputs of the above algorithms are the Fourier representation of
the decision tree and, therefore, they are
non-proper learning algorithms.

The following paper summarizes the current state of the art results in learning~$\DT_d$

\begin{center}
\begin{tabular}{|l|l|l|c|}
Adaptive/ & Deterministic/ &&References of \\
Non-Adaptive &Randomized &  Query Complexity&Alg.+Reduction\\
\hline \hline
Adaptive & Randomized & $d^42^{2d}\log^2 n$&\cite{F07}\\
\hline
Adaptive & Randomized & $d^62^{2d}\log n$&\cite{F07,BC16}\\
\hline
Adaptive & Deterministic & $2^{13d+o(d)}\log n$&\cite{KM93,BHL95}\\
\hline
Non-Adaptive & Randomized & $d^42^{2d}\log^2 n$&\cite{F07}\\
\hline
Non-Adaptive & Randomized & $d^62^{2d}\log n$&\cite{F07,BC16}\\
\hline
Non-Adaptive & Deterministic & $poly(2^d,\log n)$&OPEN\\
\hline
\end{tabular}
\end{center}

\section{Other Results}\label{OC}
In this section, we give some results for learning other Boolean classes,
arithmetic classes.

\subsection{Other Boolean Classes}
 \bigskip
{\bf $d$-MTerm}: This class is the dual of $d$-MClause.
That is, $$\mbox{$d$-MTerm=$d$-MClause$^D:=\{\overline{f(\overline{x_1},\ldots, \overline{x_n})}\ |\ f\in$ $d$-CLause$\}$.}$$

Any algorithm ${\cal A}$ that learns a class $C$ can be converted to an algorithm
${\cal B}$ that learns $C^D$ with the same query complexity.
This can be done as follows: Algorithm ${\cal B}$
runs algorithm ${\cal A}$ and for each query
$a$ that ${\cal A}$ asks, algorithm~${\cal B}$ asks the query $(\overline{a_1},\ldots, \overline{a_n})$. For each answer $b$ received by the teacher, algorithm~${\cal B}$
returns the answer $\overline{b}$ to ${\cal A}$. If algorithm ${\cal A}$ outputs $h$
then algorithm ${\cal B}$ outputs $$h^D:= \overline{h(\overline{x_1},\ldots, \overline{x_n})}.$$

\bigskip
\noindent
{\bf $d$-Term} (the dual of {\bf $d$-Clause})
We first recall the definition of $(n,d)$-universal set and then show how to use it for learning
$d$-Term.

A $d$-{\it restriction problem} \cite{AMS06,B12,NSS95} is a
problem of the following form: Given $\Sigma=\{0,1\}$, a length
$n$ and a set $B\subseteq \Sigma^d$ of assignments. Find a set
$A\subseteq \Sigma^n$ of small size such that: For any $1\le i_1<
i_2<\cdots < i_d\le n$ and $b\in B$ there is $a\in A$ such that
$(a_{i_1},\ldots,a_{i_d})=b$.

When $B=\{0,1\}^d$ then $A$ is called $(n,d)$-{\it universal set}.
The lower bound for the size $|A|$ of $(n,d)$-universal set is~\cite{KS72,SB}
\begin{eqnarray}\label{USLB}
\Omega(2^d\log n).
\end{eqnarray} Using a simple probabilistic method,
one can get the upper
bound
\begin{eqnarray}\label{USUB}
O(d2^d\log n).
\end{eqnarray}
Also, a random uniform set
of $O(2^d (d\log n+\log (1/\delta)))$ assignments in $\{0,1\}^n$ is,
with probability at least $1-\delta$, $(n,d)$-universal
set. The best known polynomial time (poly$(2^d,n)$)
construction for $(n,d)$-universal set is
of size $$2^{d+O(\log^2d)}\log n$$ \cite{NSS95}. For $d\le \log n/\log\log n$,
a $(n,d)$-universal set of size $d^{O(1)}2^d\log n$ can be constructed in polynomial
time~\cite{NSS95}.

Now consider the class of $d$-Term. Let
${\cal A}$ be an adaptive algorithm that learns this class.
Suppose the target function is the zero term and
let $S\subset \{0,1\}^n$ be the set
of queries that the algorithm asks with this target. Then $S$ must satisfy the
following property: For every $1\le i_1< i_2<\cdots < i_d\le n$
and every $b\in \{0,1\}^d$ there is $a\in S$ such that
$(a_{i_1},\ldots,a_{i_d})=b$. Otherwise, the algorithm cannot
distinguish between the zero term and the term $T=x_{i_1}^{b_1}\cdots
x_{i_d}^{b_d}$ where $x^1=x$ and $x^0=\bar x$. This is because $T$
is also zero on all the assignments in $S$. Therefore, $S$ must be
an $(n,d)$-universal set and then the query complexity of the
algorithm is at least $\Omega(2^d\log n)$.

Now it is easy to see that any $(n,d+1)$-universal set $S$ can
be used to learn non-adaptively the class $d$-Term. Just take all the positive assignments,
i.e., the assignments $a\in S$ such that $f(a)=1$, and find the
entries that have the same value in all of them.
This uniquely determines the term. Therefore
$$O(d2^d\log n)= \OPTNA(d\mbox{\rm -Term})\ge\OPTA(d\mbox{\rm -Term}) = \Omega(2^d\log n).$$
This also gives a
non-adaptive learning algorithm that asks $2^{d+O(\log^2d)}\log n$ queries and
runs in $poly(2^d,n)$ time. Therefore the class $d$-Term is
non-adaptively almost optimally learnable.

\bigskip
\noindent
{\bf XOR}: The class XOR is of size $2^{n+1}$ and therefore, by Lemma~\ref{lvlb},
any adaptive learning algorithm for XOR must ask at least $n+1$ queries. Now the trivial
algorithm that asks the $n+1$ queries $\{0\}\cup\{e_i\}_{i=1,\ldots,n}$, where $e_i$
is the assignment that is $1$ in entry $i$ and zero elsewhere, learns XOR.
Therefore, the class XOR is optimally learnable.

\bigskip
\noindent
{\bf $d$-XOR}: Since $$\mbox{$|d$-XOR$|\ge {n+1\choose d}$,}$$
by Lemma~\ref{lvlb}, the lower bound for the number of queries
for any randomized learning this class is $\Omega(d\log n)$. Uehara et al. gives
in \cite{UTW97} an adaptive algorithm that learns $d$-XOR in $O(d \log n)$ queries. Therefore
$d$-XOR is adaptively optimally learnable. Hofmeister gives in~\cite{H99}
a non-adaptive algorithm that learns $d$-XOR in $O(d\log n)$ queries.
Therefore $d$-XOR is also non-adaptively optimally learnable.

\bigskip
\noindent
{\bf $d$-Junta}: The class of $d$-Juntas is studied by Damaschke in~\cite{D98,D00,D03} and Bshouty and Costa in \cite{BC16}.
In \cite{D00}, Damaschke shows that
$$\Omega(2^d\log n)=\OPTNA(d\mbox{-Junta})=O(d^22^d+d2^d\log n).$$
He then shows that $d$-Junta is almost optimally learnable in $d$
and efficiently learnable in $n$~\cite{D98,D03}. Using
Lemma~\ref{DNtoDN} with this result, we get an algorithm that asks $2^{d+o(d)}\log n$
queries and runs in time $poly(2^d,n)$. Therefore
the class $d$-Junta is almost optimally learnable. Bshouty and Costa, \cite{BC16}, close the above gap and showed that
$$\OPTNA(d\mbox{-Junta})=O(d2^d\log n).$$
They also showed that randomness does not help improving the query complexity.
See also
other results for randomized algorithms in~\cite{BC16,D98}, optimal algorithms
for small $d$ with a constant number of rounds and bounds for
the number of rounds in~\cite{BC16,D03}.

The following is a simple adaptive learning algorithm~\cite{D00}. First ask the queries of an $(n,d)$-universal set.
Then take any two assignment $a$ and $b$ such that $f(a)\not=f(b)$.
Then find a relevant variable by a binary search on the bits that differ
between $a$ and $b$. Let $Y$ be a subset of the relevant variables that is found so far. To
find another relevant variable, we search for two assignments $a$ and $b$
that give the same values for the variables in $Y$ and $f(a)\not= f(b)$.
If no such assignments exist, then, $Y$ is the set of all the relevant variables
and then just learn the truth table over $Y$.
Otherwise, the binary search between $a$ and $b$ gives a new relevant variable.
It is easy to see that the query complexity of this algorithm is $s+d\log n$
where $s$ is the size of the $(n,d)$-universal set. This shows that
$$\Omega(2^d\log n)=\OPTA(d\mbox{-Junta})=O(d2^d\log n)$$
and, therefore, the class $d$-Junta is almost optimally adaptively learnable.

\bigskip
\noindent
{\bf $d$-MJunta}: The results in \cite{BC16,D00,D03,O10} show that
$$\Omega(2^d/\sqrt{d}+d\log n)\le \OPTA(d\mbox{-MJunta})\le O(2^d+d\log n)$$
and
$$\Omega(2^d\log n)=\OPTNA(d\mbox{-MJunta})=O(d2^d\log n).$$
Using
Lemma~\ref{DNtoDN} with the result of Damaschke in~\cite{D03},
we get a non-adaptive learning algorithm for $d$-MJunta that asks $2^{d+o(d)}\log n$
queries and runs in time $poly(2^d,n)$. Therefore
the class $d$-MJunta is almost optimally non-adaptively learnable.

\ignore{Obviously, $n$-MJunta (MDNF) can be learned in time $T(n)=2^n$ with $Q(n)=2^n$ queries.
Therefore, by Lemma~\ref{RedM}, $d$-MJunta is adaptively learnable in time $poly(n,2^d)$
and $O(d 2^d+d\log n)$ queries.}

The class of $n$-MJunta is studied in~\cite{H66,K65} where
the exact value
$$\OPTA({\rm MDNF})={n\choose \lfloor n/2\rfloor} +{n\choose \lfloor n/2\rfloor+1} $$
was found. Now, by Lemma~\ref{RedM}, $d$-MJunta is adaptively learnable in time $poly(n,$ $2^d)$
and $O(\sqrt{d} 2^d+d\log n)$ queries.
Thus, the class $d$-MJunta is adaptively almost optimally learnable.

\bigskip
\noindent
{\bf Decision Trees (DT${}_d$)}. See Section~\ref{Decision}.
\ignore{It is easy to show that any $(n,2d)$-universal set
is an equivalent testing set for DT${}_d$, and, every equivalent testing set
for $\DT_d$ is an $(n,d)$-universal set, \cite{B13d}.
Therefore, by (\ref{USLB}) and (\ref{USUB})
\begin{eqnarray}\label{DTLU}
\Omega(2^d\log n)\le \OPTA(\DT_d)\le \OPTNA(\DT_d)\le O(d2^{2d}\log n).
\end{eqnarray}
Also, an $(n,d)$-universal set is a constant testing set for $\DT_d$.

The first
deterministic adaptive algorithm for $\DT_d$ was given by Kushilevitz and Mansour in \cite{KM93}. The
algorithm asks $2^{O(d)}poly(n)$ queries. Since the number of relevant
variables in a decision tree of depth $d$ is at most $2^d$, by Lemma~\ref{BHL0},
the class $\DT_d$ is adaptively learnable with $2^{O(d)}\log n$ queries.
In particular, the class $\DT_d$ is efficiently learnable
and optimally learnable in $n$.

For non-adaptive algorithms, it follows from Feldman's results in~\cite{F07}
that $\DT_d$ is Monte Carlo non-adaptively learnable with $O(d^3 2^{2d}\log^2 n)$
membership queries, \cite{B13d}. Therefore, $\DT_d$ is MC non-adaptively efficiently learnable.
By Lemma~\ref{DNtoDN}, it follows that $\DT_d$ is MC non-adaptively
learnable with $poly(d)\cdot 2^{4d}\log n$ queries.
In particular, the class $\DT_d$ is MC optimally non-adaptively learnable in $n$.
In \cite{B13z} Bshouty gives a sketch of
a (deterministic) non-adaptive algorithm for learning BMP that in particular gives a
non-adaptive efficient learning algorithm for $\DT_d$. The result will appear in \cite{B13d}.}

\bigskip
\noindent
{\bf DNF}: This class and its subclasses are not studied in the literature for the model of exact learning
from membership queries only.

\bigskip
\noindent
{\bf Monotone DNF}: See Section~\ref{srMDNF}.

\bigskip
\noindent
{\bf CNF}: The dual class of DNF.

\bigskip
\noindent
{\bf CDNF}: This class is not studied in the literature for the model of exact learning
from membership queries only. Some non-optimal results can be achieved using
the algorithm in \cite{B95} and the reductions in Subsection~\ref{REL}.

\bigskip
\noindent
{\bf Monotone CDNF}:
The learnability of monotone CDNF is studied in \cite{BCGKT96,D96,DMP99}.
Domingo, \cite{D96}, show that the class of monotone CDNF is learnable with
a polynomial number of queries in time $s^{O(\log s)}$ where
$s$ is the size of the monotone CDNF. That is, the size of the MDNF
and MCNF of the target.
In \cite{DMP99} Domingo et al. study the learnability of the class
Read $k$-MCDNF. This is the class of monotone CDNF functions
$(f,g)$ where each variable appears
at most $k$ times in its MDNF representation $f$ and any number of times in
its MCNF representation $g$. See also~\cite{D96} for other
subclasses of monotone CDNF that are learnable from
membership queries. Bshouty et al., \cite{BCGKT96}, show that the class of MCDNF
and $O(\log n)$-CDNF are learnable
from membership queries and the NP-oracle.

\bigskip
\noindent
{\bf Boolean Multivariate Polynomial}: The efficient randomized learnability of multivariate polynomial
follows from~\cite{BM02}.
All the other algorithms in the literature require asking membership
queries from an extension field. See for example~\cite{GKS90}.

\bigskip
\noindent
{\bf XT, DFA, BMAF, ROF, BC, BF}. No results are known
for exact learning of those classes from membership queries only,
except for the trivial result that when all the variables are relevant
then $\OPTA(C)=2^n$.

\bigskip
\noindent
{\bf Boolean Halfspace (BHS):}
Heged\"{u}s, \cite{H94}, shows that BHS$(0,1)$
(with zero-one weights) are adaptively learnable in polynomial time with $O(n)$ queries.
He also gives a lower bound $\Omega(n)$ for the number of queries. Therefore,
BHS$(0,1)$ is adaptively optimally learnable. See also \cite{UTW97}.
Heged\"{u}s and Indyk, \cite{HI97}, give
a non-adaptive polynomial time learning algorithm for BHS$(0,1)$ that asks $O(n^2)$ queries.

Abboud et al., \cite{AABRS99}, show that
BHS$[k]$ (Boolean Halfspaces with weights in $\{0,1,\ldots,k\}$)
is constant-round learnable in $n^{O(k^5)}$
time and queries. They also gave the lower bound
$n^{\Omega(k)}$. Abasi et al. \cite{AAB14} give a non-adaptive algorithm
for BHS$[k]$ that asks $n^{O(k^3)}$ and a two-round algorithm that asks
$n^{O(k)}$ queries and runs in time $n^{O(k)}$. Therefore,
the class BHS$[k]$ is adaptively efficiently learnable.

Abboud et al. \cite{AABRS99} give a lower bound $\Omega(2^n/\sqrt{n})$
for BH$(-1,0,1)$ (Boolean Halfspaces with weights $\{-1,0,+1\}$). Therefore,
BH$(-1,0,1)$ is non-adaptive almost optimally learnable.
Just ask all the $\{0,1\}^n$ queries.

Uehara et al. study some restricted classes of BHS$(0,1)$,~\cite{UTW97}.

Shevchenko and Zolotykh \cite{SZ98} studied halfspace function
over the domain $\{0,1,\ldots,k-1\}^n$ when $n$ is fixed
and no constraints on the coefficients.
They gave the lower bound
$\Omega(\log^{n-2} k)$ for learning this class from
membership queries. Heged\"us~\cite{H95} proves the upper bound
$O(\log^n k/\log\log k)$. For fixed $n$, Shevchenko and Zolotykh \cite{ZS97}
gave a polynomial time algorithm (in $\log k$) for this class.
Applying Theorem 3 in \cite{H95},
the upper bound $O(\log^{n-2} k)$
for the teaching dimension of a halfspace, \cite{CZ13},
gives the upper bound $O(\log^{n-1} k/\log $ $\log k)$.

\bigskip
\noindent
{\bf MROF}.
A monotone Boolean read-once formula is a monotone formula such that
every input variable $x_i$ appears in at most one input gate. Angluin et al.
gave a polynomial time algorithm that learns MROF with $O(n^2)$ queries~\cite{AHK93,HK89}.
The best lower bound for the number of queries is the information theoretic lower bound
$\Omega(n\log n)$ that follows from Lemma~\ref{lvlb}.

Bshouty shows in \cite{Bs97} that MROF
cannot be learned efficiently in parallel (poly$(\log n)$ time).

\bigskip
\noindent
{\bf Other Classes:}
See classes of discrete functions and other classes in~\cite{B97,BHHK94,GKS93,GLR97,GT09,HI97,SZ98}.

\subsection{Classes of Arithmetic Functions}
In this section, we give few results from the literature on learning
arithmetic classes.

\bigskip
\noindent
{\bf $(r,V)$-Linear Functions ($(r,V)$-LF)}.
\ignore{where $r$
 is an integer and $V\subset \Re$. The class $(r,V)$-LF is the class
 of linear functions $f:\{0,1\}^n\to \Re$ of the form $v_1x_{i_1}+\cdots+v_{r}x_{i_{r'}}$
 where $i_1<i_2<\cdots<i_{r'}$, $r'\le r$ and $v_i\in V$, $i=1,\ldots,r$.
 The class $r$-LF is the class $(r,\{0,1\})$-LF
 and~LF is the class $n$-LF.}
The problem of learning LF is studied in~\cite{A88,C64,CM66,ER63,L64,L65,M70,SS63}.
Many authors independently proved that it is optimally learnable with
$$\Theta\left(\frac{n}{\log n}\right)$$ queries.
They do not address the time complexity, although one can show that the constructions also give simple algorithms that run in polynomial time.

The class $r$-LF
is studied in~\cite{B12b,Ca79a,Ca79b,DH00,GK00,L75,M81,TM78,UTW97}.
It is shown that
$$\OPTA(r\mbox{-LF})=\OPTNA(r\mbox{-LF})=\Theta\left(\frac{r\log (n/r)}{\log r}\right).$$
Note here that in the literature they use $\log(n/r)$ to mean $\log(2n/r)$.
In \cite{B12b}, Bshouty shows that it is optimally adaptively learnable.
The problem is still open for the non-adaptive learning.

The problem of learning $(r,\Re)$-LF is studied in~\cite{BGIKS,BM10,BM11,BM11b,Ch13,CK10,GI10,IR08,Ki12}.
Bshouty and Mazzawi, \cite{BM11}, show that
$$\OPTA((r,\Re)\mbox{-LF})=\OPTNA((r,\Re)\mbox{-LF})=\Theta\left(\frac{r\log (n/r)}{\log r}\right).$$
The results are derived from non-constructive probabilistic proofs.
All the learning algorithms for this class are either for restricted subclasses or
randomized algorithms with success probability that depend on $r$ or non-optimal.

See other subclasses in~\cite{B12b,GK00,P81,RV97}.
Similar problems are studied in other areas such as coding theory
{\rm \cite{R07}} compressed sensing {\rm \cite{K12}} Multiple
Access Channels {\rm \cite{BG07}} (e.g., adder channels {\rm
\cite{CKW06}}) and combinatorial group testing {\rm
\cite{DH00,DH06}} (e.g., coin weighing problem {\rm\cite{B12b}}).

\bigskip
\noindent
{\bf $(r,V)$-Quadratic Functions ($(r,V)$-QF)}.
\ignore{ , where $r$
is an integer and $V\subset \Re$. A $(r,V)$-QF is a function
$f:\{0,1\}^n\to \Re$ of the form $x^TAx$ where $x\in \{0,1\}^n$
and $A$ is a symmetric $n\times n$ matrix with at most $r$ non-zero entries from $V$.
The $r$-QF is the class $(r,\{0,1\})$-QF.}
This problem is equivalent to learning a weighted graph from additive queries, \cite{GK98},
where, for an additive query, one chooses a set of vertices and asks the sum of the weights of
edges with both ends in the set.

The $r$-QF was studied in~\cite{A88,CK08,CK10,G98,GK98,GK00,M10,RS07}.
The $(r,V)$-QF for different $V\subseteq\Re$ was studied in~\cite{BM11,BM11b,BM12,Ch13,CK10}.
Bshouty and Mazzawi,~\cite{BM11}, proved that
$$\OPTNA(r\mbox{-}\QF)=\OPTNA((r,\Re)\sline\QF)=\Theta\left(\frac{r\log n}{\log r}\right).$$
The results are derived from non-constructive probabilistic proofs.
For the positive real numbers $\Re^+$, Bshouty and Mazzawi gave in \cite{BM12} a polynomial time
algorithm that adaptively learns the class $(r,\Re^+)\sline\QF$ with $O(r\log n/$ $\log r+r\log\log r)$ queries.
This is the only known deterministic adaptive algorithm that runs in polynomial time. Choi,~\cite{Ch13}, gave a polynomial time randomized
adaptive learning algorithm for $(r,\Re)\sline\QF$ that asks $O(r\log n/\log r)$ queries.

Bshouty and Mazzawi extended some of the above results to multilinear forms
of constant degree~\cite{BM10}.

\bigskip
\noindent
{\bf Multivariate Polynomial}: This class has been extensively studied in the literature.
Ben-Or and Tiwari~\cite{BT88} gave the first deterministic non-adaptive polynomial time
learning algorithm for sparse multivariate polynomial over a large field
with an optimal number of queries. See also~\cite{GK91,GKS90,GKS90b,KL88,KS01}.

For identity testing and zero testing of sparse
multivariate polynomials see~\cite{B05,BHS08,GKS90b,H07,KS96,KS01,W94} and references therein.

\bigskip
\noindent
{\bf Multiplicity Automata Function}: This class was first defined and
studied in~\cite{BBBKV00}. It is efficiently learnable
from queries with a randomized MC algorithm~\cite{BBBKV00}.

\bigskip
\noindent
{\bf Arithmetic Circuit and Arithmetic Formula}:
In \cite{V79} Valiant suggests an algebraic analog of P vs. NP,
the VP vs. VNP problem. A multivariate polynomial family $\{p_n(x_1,\ldots,x_n)\}_{n\ge 1}$ is
in VP if there exists a constant $c > 0$ such that for all $n$, deg$(p_n)\le n^c$
and $p_n$ has a circuit of size bounded by $n^c$. Polynomial family
$\{q_n\}_{n\ge 1}$ is in VNP if there exists a family $\{p_n\}\in
$VP  such that for every~$n$
$$q_n(x_1,\ldots,x_n)=\sum_{y\in\{0,1\}^n} p_{2n}(x_1,\ldots,x_n,y_1,\ldots,y_n).$$
Valiant shows in \cite{V79} that permanent is complete for VNP,
i.e., for every polynomial family $\{q_n\}_{n\ge 1}$ in VNP, there
is a constant $c > 0$ such that for every $n\ge 1$, $q_n$ can be
expressed as permanent of a matrix of size $n^c\times n^c$. It is
believed that VP$\not=$VNP. This remains an outstanding open
problem.

In \cite{AV08}, Agrawal and Vinay show that if there exists a
deterministic polynomial time zero testing for arithmetic circuits
of degree $d$ and depth $4$ then there exists a polynomial family
$\{q_n\}_{n\ge 1}$, computable in exponential time, that is not in
VP. So an efficient deterministic zero testing for such circuits
leads to a proof of circuit subexponential lower bounds that may
be beyond our proof techniques.

Kabanets and Impagliazzo show in \cite{KI03} that even if the zero
testing algorithm gets the arithmetic circuit as an input (white
box) if there exists a deterministic polynomial time algorithm for
zero testing for VP then either NEXP$\not\in$P/poly or
VP$\not=$VNP. Therefore, any deterministic algorithm implies
solving outstanding open problems in complexity. See
\cite{AS09,SY10} for other negative results.

On the other hand, the following Schwartz-Zippel lemma, \cite{S80,Z79},
gives a very simple MC randomized optimal zero testing
algorithm for any arithmetic circuit with a bounded degree
\begin{lemma} {\bf (Schwartz-Zippel)}
Let $f\in F[x_1,\ldots,x_n]$ be any non-zero polynomial of degree $d$
and $S\subset F$. Then for $y_1,\ldots,y_n$ selected randomly
uniformly from $S$ we have
$$\Pr_{y\in S^n}[f(y_1,\ldots,y_n)\not=0]\ge 1-\frac{d}{|S|}.$$
\end{lemma}

For the deterministic identity testing of arithmetic circuits of depth
$3$, restricted depth $4$ circuits, circuits that compute sparse
polynomials and other restricted circuits see the results in
\cite{AMV11,AS09,ASS13,BHLV09,CDGK91,CKW11,KMSV10,KS08,KS09,S13,S09,S14,SS11,SY10,SV09,SV11,V16}
and references therein. Some other results in the literature
investigate the problem of minimizing the number of random bits
used for identity testing. See for example \cite{B05,BHS08,KS01}.

\bigskip
\noindent
{\bf Arithmetic Read-Once Formulas (AROF):}
Arithmetic Read-Once Formula is a formula where
each variable appears at most once.
In \cite{BHH95} Bshouty et al. gave an
MC randomized polynomial time algorithm for
AROF (with the division operation) over a large enough field $F$. In
\cite{BC98} Bshouty and Cleve gave a polynomial time
(poly(log)) randomized parallel
algorithm for this class. In \cite{BB98}, Bshouty and Bshouty
extended the result of~\cite{BHH95} to include
the exponentiation operation. Shpilka and Volkovich in \cite{SV10} gave a
deterministic algorithm for learning depth $d$
AROF in time $n^{O(d)}$.
In \cite{SV09} Shpilka and Volkovich gave a deterministic
learning algorithm for AROF that asks $n^{\log n}$ queries.
They also studied the class of sum of $k$ AROFs.
Recently, Volkovich gave in~\cite{MV16} a polynomial time algorithm
for learning any AROF.

\bigskip
\noindent
{\bf Other Classes:} See other results and other classes in \cite{A88,AS09,CKW11,KS09,RS05,S09,S13,S14,SV09,SS11,SY10,SV10,SV11}
and references therein.

\section{Non-Honest Teacher}
Although the aim of this survey is to summarize the results of learning from an honest
teacher, we feel a need to give here some of the models of non-honest teacher and some results.

\subsection{Models of Non-Honest Teacher}\label{TT}

In this survey, the teacher model is the {\it honest teacher model} where
with a query $d\in X$,
the teacher answers $f(d)$.

For non-honest teacher, there are many models.
One can consider a {\it persistent} teacher~\cite{AS94,P94}
or a {\it non-persistent} teacher.
For persistent teacher (or permanently faulty~\cite{P94})
if the answer to the query $d$ is $y$ then no matter how
many times the learner asks the same query the answer will be $y$.
A non-persistent teacher is a teacher that is not persistent.
In the literature the following non-honest teacher models
are considered (each one can be either persistent or nonpersistent):
\begin{enumerate}\setlength\itemsep{.5em}

\item {\bf Incomplete Model} \cite{AS94}: The incomplete
teacher, with a query $d$, answers $f(d)$
with probability $p$ and answers ``$?$'' (I DON'T KNOW) with
probability $1-p$. In the persistent model, repeated queries to $d$ will give the same answer with probability $1$. In the non-adaptive model, the learner knows $p$ or some upper bound for $p$.

\item {\bf Malicious Model} \cite{KL93,R61,Sa91,V85}: (Also called
{\it random error} \cite{R61} and {\it classification noise} \cite{J97}) The malicious teacher,
with a query $d$,
answers $f(d)$ with
probability $p$ and gives an arbitrary/random wrong answer with probability $1-p$.
The learner knows $p$ or some upper bound for $p$.

\item {\bf Limited Incomplete Model} \cite{AKST97}: The
limited incomplete teacher gives
answers ``$?$'' (I DON'T KNOW) to at most $m$ queries of its choice.
In the non-adaptive model, the learner knows $m$ or some upper bound for $m$.

\item {\bf Limited Malicious Model} \cite{AKST97,U76}: (Also
called the {\it constant number of error model}~\cite{AD91,RMKWS}) The
limited malicious teacher gives
arbitrary/random wrong answers to at most $m$ queries of its choice.
The learner knows~$m$ or some upper bound for $m$.

\item {\bf Prefix-Bounded Error Fraction Model}~\cite{P87b}:
(Also called {\it linearly} {\it bounded model}~\cite{AD91})
In the adaptive model, the teacher after $t$ queries can
give at most $pt$ wrong answers. In the $r$-round model, at each round with $T$
queries and for any $1\le t\le T$,
the learner can give $pt$ wrong answers to the first $t$ queries in this round.
The learner knows $p$ or some upper bound for $p$.

\item {\bf Globally Bounded Error Fraction Model}
\cite{P87b}:
In the adaptive model, if the algorithm asks $T$ queries
then the teacher can give at most $pT$ wrong answers.
In the $r$-round model, at each round with $T$
queries, the learner can give at most $pT$ wrong answers.
The learner knows $p$ or some upper bound for $p$.

Notice that in the globally bounded error fraction model the first $pT$ queries can be all wrong while in the prefix-bounded error fraction model only $p^2T$ queries of the first $pT$ queries can be wrong.

\item {\bf Incomplete Prefix-Bounded Error Fraction Model}:
In the adaptive model, the teacher after $t$ queries can
give at most $pt$ ``?'' answers. In the $r$-round model, at each round with $T$
queries and for any $1\le t\le T$,
the learner can give $pt$ ``?'' answers to the first $t$ queries in this round.
In the non-adaptive model, the learner knows $p$ or some upper bound for $p$.

\item {\bf Incomplete Globally Bounded Error Fraction Model} \cite{B12b}:
In the adaptive model, if the algorithm asks $T$ queries
then the teacher can give at most $pT$ ``?'' answers.
In the $r$-round model, at each round with $T$
queries, the learner can give at most $pT$ ``?'' answers.
In the non-adaptive model, the learner knows $p$ or some upper bound for $p$.

\item {\bf $E$-Sided Error Models}: (Also called {\it half-error} \cite{P02},
or {\it one-sided error} \cite{RMKWS}, for Boolean functions) Can be defined for any one of the above
models where the wrong or``?'' answers only applied when $f(d)$ is in some set $E\subset R$.

\end{enumerate}
For the persistent model we define the output hypothesis to be equivalent to the target function if it agrees with the target function on all the elements of the domain except the ones for which the teacher answer ``?'' or gave a wrong answer.

\subsection{Some Results in Learning with Non-honest Teacher}

In this subsection, we give some results of learning with a non-honest teacher.

Adaptively learning Var$=\{x_1,\ldots,x_n\}$ in non-honest teacher model is equivalent to
the problem of ``searching with lies''~\cite{U76}. Ulam~\cite{U76}
proposed the following game.
Someone thinks of a number between one and one million (which is just less
than $2^{20}$). Another person is allowed to ask up to twenty questions, to each of
which the first person is supposed to answer only yes or no. Obviously, the number
can be guessed by asking first: is the number in the first half-million? Then
again reduce the reservoir of numbers in the next question by one-half, and so
on. Finally, the number is obtained in less than $\log_2 1 000 000$ questions.
The number $h$ corresponds to the target variable $x_h$ in the class Var
and each question ``Is $h\in H$?'' corresponds to the query $(a_1,\ldots,a_{1000000})$
where $a_i=1$ if and only if $i\in H$.

Ulam asked the following question: Now suppose one
were allowed to lie once or twice, then how many questions would one need to
get the right answer? This problem is equivalent to learning the class Var
in the limited malicious model. R\'{e}nyi~\cite{R61} asked a similar question
and therefore, the game is called R\'{e}nyi-Ulam game.

This problem is completely solved with an asymptotically optimal
number of queries in the limited malicious model~\cite{A88,G90,P87,S92}.
See also the references in \cite{P02} for results when the number of lies is small.
Learning this class in two-round is studied in~\cite{CM99,CM00,CMV00}.

The problem is solved with an asymptotically optimal
number of queries in the linearly bounded model~\cite{AD91,P87b,SW92}.
It is also noted by several authors that finding a non-adaptive algorithm
in this model is equivalent to constructing a $t$-error correcting code~\cite{P02}.

See the survey in \cite{P02} for results in other models of non-honest teacher.

For learning $d$-MClause and $s$-term $r$-MDNF with non-honest teacher see
\cite{AA13,AD91,C10,C13,CD08,CH08,CJBJ13,DH00,LWYGW09,RMKWS,W06} and references therein.

\setenumerate[1]{label=\thesection.\arabic*.}
\setenumerate[2]{label*=\arabic*.}

\section{Problems and Open Problems }

In this section, we give some problems and open problems
$$ $$

\setcounter{section}{1}
{\bf Section \ref{Int}}
\begin{enumerate}\setlength\itemsep{.5em}
\item In real life problems, the target function may
change in time. Define a realistic learning model for learning functions
that change in time.

\item In the results of this survey and almost all papers in the literature, the space complexities of learning algorithms are polynomial in $\OPT_A(C)$ which, for many classes $C$, is exponential in $n$ and/or
    other parameters that depend on the class. It is interesting to investigate learning algorithms that use small space complexity.

\item
It is interesting to minimize the number of random bits used
in randomized learning algorithms. See for example item \ref{rb}

\item It is interesting to study the exact learnability of a
random function in a class~$C$ from membership queries. See, for example, some models in~\cite{JLSW11,JS06,T96}.

\item An LV randomized {\it non-adaptive} algorithm with query complexity $Q$
of complexity $T$ is an algorithm that asks
{\it at most} $Q$ queries and runs in {\it expected time} $T$. So any LV randomized non-adaptive
algorithm is deterministic in choosing the queries. We suggest the following definition
that allows expected query complexity in non-adaptive learning algorithms:
A {\it weak LV randomized non-adaptive algorithm} with complexity $T$
is a non-adaptive algorithm that (1) generates queries
that are independent of the answers to the previous queries.
(2) Finds the target function with probability $1$. (3) The expected number
of queries is $Q$ and the expected time is $T$.

\item In this survey, we have shown some results for the testing problems.
Some of those results are not true for LV/MC randomized algorithms. For example,
in deterministic algorithms, the query complexity of non-adaptive learning
is equal to the minimum size equivalent test.
For randomized algorithms, one can non-adaptively equivalent test
the class XOR with $O(\log 1/\delta)$ random queries whereas learning XOR by a randomized algorithm takes at least $n$ queries.
It is interesting to study MC and LV randomized equivalent test and other types of tests in the adaptive and non-adaptive model.

\item Investigate testing in the deterministic/randomized $r$-round model.

\item There are very few results in the literature on parallel learning
from membership queries. That is, learning
in $poly(\log)$ time. Study parallel learning.

\item To the best of my knowledge, all the Monte Carlo learning algorithms in the
literature ignore minimizing the effect of the success probability $1-\delta$
in computing the number of queries. Some of the results even ignore $\delta$ by assuming
that it is constant. It is interesting to investigate the role of $\delta$ in
the query complexity.

\item We say that a non-adaptive algorithm is {\it strongly nonadaptive} if the queries are constructed by different learners (one query for each learner) without any communication between them. It is interesting to study this model or any model with minimum communication between the learners.

\end{enumerate}

\setcounter{section}{2}
{\bf Section \ref{Se2}}
\begin{enumerate}\setlength\itemsep{.5em}
\item In the bound $$\frac{2\cdot \ETD(C)}{\log \ETD(C)}\log|C| \ge \OPT_\A(C)\ge \max(\ETD(C),\log |C|),$$
find some conditions on $C$ for which tighter bounds can be obtained.

\item Many lower bounds in the literature for $\OPT(C)$ are based on finding a subset of functions $C'\subseteq C$
such that for each membership query there is an answer that
eliminates at most small fraction of the functions.
The best possible bound that one can get using this technique is denoted by ${\rm DEN}(C)$.
In \cite{BM17} Bshouty and Makhoul show that $\ETD(C)\ge {\rm DEN}(C)-1$.
Find a new combinatorial measure that is a lower bound for $\OPT(C)$ and exceeds $\ETD(C)$.

\item The algorithm in (\ref{algforopt}) runs in time $n^{O(2^h)}$ where $h$ is the depth of the tree. Find an algorithm with a better exponential complexity.

\item Find a non-adaptive learning algorithm that runs in time $poly(|X|,|C|)$
and learns $C$ using at most
$(c\ln|C|)\OPTNA(C)$ queries for some $c<2$.

\item Study the above bounds
and find approximation algorithms for randomized adaptive learning.

\item Study bounds and find approximation algorithms for adaptive learning of
classes with small VC-dimension.

\item Study the above bounds
and find approximation algorithms for $r$-round learning.

\item Is NP-oracle enough for deterministic/randomized optimal learnability? What other oracle gives learning with minimum number of membership queries?

\item In \cite{BGHM96} some techniques were used in the model
of exact learning from membership and equivalence queries to minimize
the number of equivalence queries. Can those be used to find more query-efficient
algorithms?

\item Study the above bounds
and find approximation algorithms for classes with small extended teaching dimension.

\item Study bounds for LV and MC randomized algorithms.
\end{enumerate}

\setcounter{section}{3}
{\bf Section~\ref{Reduction}}
\begin{enumerate}\setlength\itemsep{.5em}
\item The reductions in subsection~\ref{AtoA} are for adaptive and non-adaptive learning.
It is interesting to find reduction results for $r$-round deterministic and randomized algorithms.

\item The reductions in
subsection~\ref{AtoA} are for the number of relevant variables.
Find reductions for other parameters, for example, the number
of terms (e.g. for MP or MDNF).

\item Find reductions that give algorithms that are optimally learnable or almost optimally
learnable from $r$.

\item Lemma~\ref{RedEQ} is implicitly
used for some of the results in the literature for learning some classes.
For example, the Halving algorithm is an algorithm
that asks equivalence query with ``Majority$(C')$'' at each stage,
where $C'\subseteq C$ are the functions in $C$ that are consistent
with the counterexamples seen so far.
Lemma~\ref{LBo} is just a reduction from the Halving algorithm.
It is interesting to study learnability of the classes
mentioned in this survey with this technique.

\item There are many
polynomial time exact learning algorithms from membership and
equivalence queries in the literature for classes mentioned in
this survey and others.
See~\cite{A87,A87b,AFP92,AP91,B95,B97,B98,BBBKV00,BHH95a,BR95,K97,SS96,SST10}.
It is interesting to study the reduction of those algorithms
to learning from membership queries only when some of the parameters of the
class is restricted. For example, can
Angluin-Frazier-Pitt learning algorithm for conjunctions of horn
clauses,~\cite{AFP92}, be changed to learning from membership queries when the number of terms is
bounded by $d$ or/and the size of each clause is bounded by $k$.

\item Let $H$ be a family of functions $h:[n]\to [q]$.
For $d\le q$ we say that $H$ is an $(n,q,d)$-{\it perfect hash family} ($(n,q,d)$-PHF)
\cite{AMS06} if for every
subset $S\subseteq [n]$ of size $|S|=d$ there is a {\it hash
function} $h\in H$ such that $h|_S$ is injective (one-to-one) on~$S$, i.e.,
$|h(S)|=d$. In \cite{B14b} it is shown that
for $q\ge 2d^2$. There is a $(n,q,d)$-PHF of size
$O\left({d^2\log n}/{\log(q/d^2)}\right)$
that can be constructed in time $O(qd^2n\log n/\log(q/d^2))$.
This construction is used for many reduction in learning. It is known that there
is a $(n,O(d^2),d)$-PHF of size $O\left({d\log n}\right)$. Finding a polynomial time construction for $(n,O(d^2),d)$-PHF of such size improves the query complexity of many reductions.
\end{enumerate}

\setcounter{section}{4}
{\bf Section~\ref{GTest}}
\begin{enumerate}\setlength\itemsep{.5em}
\item Non-adaptive randomized algorithms have been proposed in~\cite{BBKT95,BKB95,DH06,ER63,H00,HL01}.
The following models are studied in the literature for constructing the $m\times n$ random test matrix $M$
\begin{itemize}
\item {\it Random incidence design} (RID algorithms). The entries in $M$ are chosen randomly and independently to be $1$ with probability $p$ and $0$ with probability $1-p$.
\item {\it Random $r$-size design} (RrSD algorithms). The rows in $M$ are chosen randomly and independently from the set of all vectors $\{0,1\}^n$ of weight~$r$.
\item {\it Random $k$-set design} (RkSD algorithms) The columns in $M$ are chosen randomly and independently from the set of all vectors $\{0,1\}^m$ of weight~$k$.
\end{itemize}
Find lower and upper bounds for the constant in $\theta(d\log n)$ of the number of membership queries for the above non-adaptive learning algorithms.

\item Find a polynomial time $O(1)$-round algorithm for learning $d$-MClause
that asks $O(d\log n)$ queries.

\item Find a deterministic non-adaptive learning algorithm for $d$-MClause
that asks $O(d^2\log n$ $/\log d)$ queries.

\item A construction of a $d$-disjunct matrix is called {\it globally explicit} construction if it is deterministic polynomial time in the size of the construction. A {\it locally explicit construction} is a construction where one can find any entry in the construction
in deterministic poly-log time in the size of the construction.
In particular, a locally explicit construction is also globally explicit.
The constructions in the literature for $d$-disjunct matrices are globally explicit constructions.
Find a locally explicit construction of $d$-disjunct matrix of size $O(d^2\log n)$.

\item  There are few results in the literature about learning $d$-MClause when
$d$ unknown to the learner. It is interesting to study this problem.

\item Let ${\cal F}$ be a set of $n$ functions $f:X\to \{0,1\}$. Define $d$-MClause(${\cal F}$) the set of all functions $f_1\vee f_2\vee \cdots \vee f_{d'}$ where $d'\le d$ and $f_i\in {\cal F}$ for all $i=1,\ldots,d'$. Study the learnability of the class $d$-MClause(${\cal F}$).
\item Study the learnability of the class of monotone clauses with constant number of negated variables.

\item\label{rb} Any deterministic algorithm for non-adaptive learning $d$-MClause has query complexity $\Omega(d^2\log n/\log d)$ while there is a Monte Carlo non-adaptive learning algorithm
     that asks $O(d\log n)$ queries only and uses $O(dn\log n)$ random bits.
     What is the tradeoff between the number
    of random bits and the query complexity?
\end{enumerate}

\setcounter{section}{5}
{\bf Section~\ref{srMDNF}}
\begin{enumerate}\setlength\itemsep{.5em}
\item Find strong learning algorithms for $s$-term $r$-MDNF with the parameter $r$
or/and $s$.

\item Many results in the literature for learning sub-classes of $s$-term $r$-MDNF
are query-efficient, but are not time-efficient. It is interesting
to find polynomial time learning algorithms for those classes.

\item Find a non-adaptive efficient learning algorithm for the class $s$-term $r$-MDNF
when $r=\omega(1)$.

\item Angluin and Chen gave in \cite{AC08} a polynomial time $5$-round Las Vegas
algorithm for learning $s$-term $2$-MDNF that asks $O(s\log n +\sqrt{s}\log^2 n)$
queries. Can this class be learned in $O(1)$-round with $O(s\log n)$ queries?

\item Find $\OPT_{R-\RAD}($$s$-term 2-MDNF) for $R=2,3,4$.

\item Give an optimal learning algorithm for $s$-term $r$-MDNF for constant $r>2$.

\item The class of Read-Once $2$-MDNF is equivalent to learning matchings~\cite{ABKRS04}.
Alon et al. gave bounds for deterministic, randomized and $r$-round learning this class.
Extend the results to other related classes such as Read-Once $r$-MDNF,
Read-Twice $2$-MDNF and Read-Once $2$-DNF.
\end{enumerate}

\setcounter{section}{6}
{\bf Section~\ref{Decision}}
\begin{enumerate}\setlength\itemsep{.5em}
\item We show that
$$\Omega(2^d\log n)\le \OPT_\A(C)\le \OPT_\NA(C)\le O(d2^{2d}\log n).$$
Close the gap between the lower and upper bound.

\item What are the query complexities of the randomized learning algorithms for $\DT_d$ in~\cite{BM02,SS96}?

\item The deterministic adaptive algorithm of Kushilevitz-Mansour~\cite{KM93} asks
$O(2^{10d}$ $n\log n)$ queries. Find a more query-efficient algorithm.

\item Find a proper learning algorithm for $\DT_d$.
Can $\DT_d$ be learned from $\DT_{poly(d)}$?

\item Let $X$ be a finite set and $Y$ be any set. One of the important
representations of functions $f:X^n\to Y$ is
decision tree over the alphabet $Y$ with output $X$. A {\it decision tree} over $Y$
with output $X$ is defined as follows: The
constant functions $y\in Y$ are decision trees. If $f_i$ is a decision trees
for $i=1,\ldots,t$ and $S_1,\ldots,S_t$ is a partition of $X$
then, for all $j=1,\ldots,n$,
\begin{eqnarray}\label{ifelse2}
\mbox{``$f'=$[if $x_j\in S_1$ then $f_1$
else if $j\in S_2$ then $f_2$ $\cdots$}\nonumber\\
\mbox{\ \ \ \ \ \ \ \   else if $j\in S_t$ then $f_t$]"}
\end{eqnarray}
is a decision tree (can also be expressed as $f'=[x_i\in S_1]
f_1+\cdots+ [x_i\in S_t] f_t$. Here $[x\in S]=1$ if $x\in S$ and $0$ if $x\not\in S$.
Every decision tree $f'$ can be represented as
a tree $T(f')$. If $f'= y$ for some $y\in Y$ then $T(f')$ is a node
labeled with $y$. If $f'$ is as in (\ref{ifelse2}), then $T(f')$ has a root labeled with $x_i$
and has $t$ outgoing edges. The $i$th edge is labeled with $S_i$ and
is pointing to the root of $T(f_i)$.
See for example the decision tree of tastes preference in Figure~\ref{Tree01}.

Find an efficient learning algorithm for decision trees over large alphabet.

\item Find an efficient deterministic non-adaptive learning algorithm for~$\DT_d$.

\item Study the learnability of $\DT_{d,s}$, MDT$_{d,s}$ and DL.
\end{enumerate}

\setcounter{section}{7}
{\bf Section~\ref{OC}}
\begin{enumerate}\setlength\itemsep{.5em}
\item The randomized MC query complexity of $d$-Term
is less than the deterministic query complexity.
It is interesting to study $r$-round LV randomized algorithms for this class.

\item Close the gap between the upper bound and the lower bound of
$\OPTNA(d$-MJunta$)$.

\item Study the learnability of the subclasses of DNF and CDNF defined in survey.

\item Find $\OPT_A(\CDNF)$ and $\OPT_A(\MCDNF)$ for adaptive and non-adaptive algorithms.

\item Study the learnability of the classes XT, DFA and BMAF.

\item Find $\OPTNA(\BHS[k])$. The current upper bound is $n^{O(k^3)}$, and the lower
bound is $n^{\Omega(k)}$.

\item Find $\OPTA($MROF$)$.

\item Study the learnability of the conjunction and disjunction of two MROF.

\item Find a non-adaptive algorithm for $r$-LF (resp. $r$-QF) with $O(r\log n/\omega(1))$ queries.

\item Find a randomized algorithm for $(r,\Re)$-LF (resp. $(r,\Re)$-QF) with an optimal number of queries
with success probability $1-1/poly(n)$.

\item Find a deterministic efficient learning algorithm for multiplicity automata function.

\end{enumerate}

\end{document}